%% file: crowd.tex
\documentclass[11pt,letter]{article}

\usepackage[margin=1in]{geometry}
\usepackage{amssymb,amsthm,amsmath,amssymb,wrapfig,dsfont}
\usepackage[dvipsnames]{xcolor}

\usepackage{times}
\usepackage{mathtools}
\usepackage{dsfont} 
\usepackage{amssymb}
\usepackage{graphicx}
\usepackage{url}
\usepackage{setspace}
\usepackage{algorithm,algorithmic}
\usepackage{hyperref}
\usepackage{framed}
\usepackage{xcolor}
\usepackage{soul}
\usepackage{natbib}
\usepackage[autostyle]{csquotes}  
\usepackage{verbatim}
\usepackage{enumitem}
\usepackage{varwidth}
\usepackage{graphicx}
\usepackage{subcaption}
\usepackage{wrapfig}
\usepackage{breakcites}
\usepackage{natbib}

\usepackage[algo2e,ruled]{algorithm2e}

\usepackage{hyperref}
\hypersetup{
     colorlinks   = true,
     urlcolor    = teal,
	 citecolor = teal,
	 linkcolor = teal,
	 breaklinks=true 
}

\mathchardef\dash="2D
\newtheorem{theorem}{Theorem}[section]
\newtheorem{lemma}[theorem]{Lemma}

\renewcommand{\O}{\mathcal{O}}
\newcommand{\F}{\mathcal{F}}

\DeclareMathOperator{\E}{\mathbb{E}}

\newcommand{\poly}{\mathrm{poly}}
\newcommand{\maj}{\mathrm{Maj}}
\newcommand{\smaj}{\mathrm{Maj{\dash}size}}
\newcommand{\err}{\mathrm{err}}

\newcommand{\Maj}{\mathit{MAJ}}

\renewcommand{\bar}[1]{\overline{#1}}

\usepackage{enumitem}
\setenumerate{noitemsep,topsep=0pt,parsep=0pt,partopsep=0pt}
\setitemize{noitemsep,topsep=0pt,parsep=0pt,partopsep=0pt}

\newcommand{\load}{\lambda}
\newcommand{\query}{\Lambda}
\newcommand{\golden}{\Gamma}
\renewcommand{\ell}{g}

\newcommand{\D}{D}
\newcommand{\DX}{{D_{\mid\X}}}
\newcommand{\X}{\mathcal{X}}
\newcommand{\Y}{\mathcal{Y}}

\title{Efficient PAC Learning from the Crowd}

\author{Pranjal Awasthi\thanks{Rutgers University, \texttt{pranjal.awasthi@rutgers.edu}} \and
Avrim Blum\thanks{Carnegie Mellon University, \texttt{avrim@cs.cmu.edu}. Supported in part by NSF grants CCF-1525971 and CCF-1535967.  This work was done in part while the author was visiting the Simons Institute for the Theory of Computing.} \and
Nika Haghtalab\thanks{Carnegie Mellon University, \texttt{nhaghtal@cs.cmu.edu}. Supported in part by NSF grants CCF-1525971 and CCF-1535967 and a Microsoft Research Ph.D. Fellowship.  This work was done in part while the author was visiting the Simons Institute for the Theory of Computing.} \and 
Yishay Mansour\thanks{Blavatnik School of Computer Science, Tel-Aviv University, \texttt{mansour@tau.ac.il}. This work was done while the author was at Microsoft Research, Herzliya.
Supported in part by a grant from the Science Foundation (ISF), by a grant from United States-Israel Binational Science Foundation (BSF), and by The Israeli Centers of Research Excellence (I-CORE) program (Center No. 4/11).}}

\date{}

\begin{document}

\maketitle

\allowdisplaybreaks

\begin{abstract}
\input{abstract}

\end{abstract}

\input{intro}

\input{model}

\input{results}
\input{warmup}

\input{boost}

\input{anyAlpha}
\input{community}

\bibliographystyle{plainnat}
\bibliography{../../../Nika_bib}
\clearpage

\appendix
\input{app_related}
\input{app_super-sample}

\input{app_size}
\input{app_bernstein}
\input{app_probability}
\input{anyAlphaApp}

\end{document}

%% file: abstract.tex
In recent years crowdsourcing has become the method of choice for gathering labeled training data for learning algorithms. Standard approaches to crowdsourcing view the process of acquiring labeled data separately from the process of learning a classifier from the gathered data. This can give rise to computational and statistical challenges. For example, in most cases there are no known computationally efficient learning algorithms that are robust to the high level of noise that exists in crowdsourced data, and efforts to eliminate noise through voting often require a large number of queries per example.

In this paper, we show
how by interleaving the process of labeling and learning, we can
attain computational efficiency with much less overhead in the labeling cost.
In particular, we consider the {\em realizable} setting where there exists a true target function in $\F$ and consider a pool of labelers.
When a noticeable fraction of the labelers are \emph{perfect}, and the rest  behave arbitrarily, we show that any $\F$ that can be efficiently learned in the traditional {\em realizable} PAC model can be learned in a computationally efficient manner by querying the crowd, despite high amounts of noise in the responses. Moreover, we show that this can be done while each labeler only labels a constant number of examples and the number of labels requested per example, on average, is a constant.
When no perfect labelers exist, a related task is to find a set of the labelers which are \emph{good} but not perfect.
We show that we can identify  all good labelers, when at least the majority of labelers are good.

%% file: intro.tex
\section{Introduction}
Over the last decade, research in machine learning and AI has seen tremendous growth, partly due to the ease with which we can collect and annotate massive amounts of data across various domains. This rate of data annotation has been made possible due to crowdsourcing tools, such as {Amazon Mechanical Turk\texttrademark}, that facilitate individuals' participation  in a labeling task.
In the context of classification, a crowdsourced model uses a large pool of workers to gather labels for a given training data set  that will be used for the purpose of learning a good classifier.
Such learning environments that involve the crowd give rise to a multitude of design choices that do not appear in traditional learning environments.
These include:
How does the goal of learning from the crowd differs from the goal of annotating data by the crowd?
What challenges does the high amount of noise typically found in curated data sets~\citep{wais2010towards, kittur2008crowdsourcing, ipeirotis2010quality} pose to the learning algorithms? How do learning and labeling processes interplay? How many labels are we willing to take per example?
And, how much load can a labeler handle?

In recent years, there have been many exciting works addressing various theoretical aspects of these and other questions~\citep{slivkins2014online}, such as reducing noise in crowdsourced data~\citep{dekel2009vox}, task assignment~\citep{BKS13, tran2014efficient} in online or offline settings~\citep{karger2014budget}, and the role of incentives~\citep{vaughan2013adaptive}. In this paper we focus on one such aspect, namely, {\em how to efficiently learn and generalize from the crowd with minimal cost?}
The standard approach is to view the process of acquiring labeled data through crowdsourcing and the process of learning a classifier in isolation.
In other words, a typical learning process involves collecting data labeled by many labelers via a crowdsourcing platform followed by running a passive learning algorithm to extract a good hypothesis from the labeled data. 
As a result, approaches to crowdsourcing focus on getting high quality labels per example and not so much on the task further down in the pipeline. Naive techniques such as taking majority votes to obtain almost perfect labels have a cost per labeled example that scales with the data size, namely $\log (\frac m \delta)$ queries per label where $m$ is the training data size and $\delta$ is the desired failure probability. This is undesirable in many scenarios when data size is large. Furthermore, if only a small fraction of the labelers in the crowd are perfect, such approaches will inevitably fail. 
An alternative is to feed the noisy labeled data to existing passive learning algorithms. However, we currently lack computationally efficient PAC learning algorithms that are provably robust to high amounts of noise that exists in crowdsourced data. Hence separating the learning process from the data annotation process results in high labeling costs or suboptimal learning algorithms.

In light of the above, we initiate the study of designing efficient PAC learning algorithms in a crowdsourced setting where learning and acquiring labels are done in tandem.
We consider a natural model 
of crowdsourcing and ask the fundamental question of whether efficient learning with little overhead in labeling cost is possible in this scenario. We focus on the classical PAC setting of~\cite{Valiant:84} 
where there exists a true target classifier $f^*\in \F$ and the goal is to learn $\F$ from a finite training set generated from the underlying distribution.
 We assume that one has access to a large pool of labelers that can provide (noisy) labels for the training set. 
We seek algorithms that run in polynomial time and produce a hypothesis with small error. 
We are especially interested in settings where there are computationally efficient algorithms for learning $\F$ in the consistency model, i.e. the realizable PAC setting.
Additionally, we also want our algorithms to make as few label queries as possible, ideally requesting a total number of labels that is within a constant factor of the amount of labeled data needed in the realizable PAC setting. We call this $O(1)$ \emph{overhead} or \emph{cost per labeled example}.
Furthermore, in a realistic scenario each labeler can only provide labels for a constant number of examples, hence we cannot ask too many queries to a single labeler. We call the number of queries asked to a particular labeler  the {\em load} of that labeler.

Perhaps surprisingly, we show that when a noticeable fraction of the labelers in our pool are \emph{perfect} all of the above objectives can be achieved simultaneously. That is, {\em if $\F$ can be efficiently PAC learned in the realizable PAC model, then it can be efficiently PAC learned in the noisy crowdsourcing model with a constant cost per labeled example}. In other words,
{the ratio of the number of label requests in the noisy crowdsourcing model to the number of labeled examples needed in the traditional PAC model with a perfect labeler is a constant and does not increase with the size of the data set.}
Additionally, each labeler is asked to label only a constant number of examples, {i.e., $O(1)$ load per labeler.} Our results also answer an open question of~\cite{dekel2009vox} regarding the possibility of efficient noise robust PAC learning by performing labeling and learning simultaneously. 
When no perfect labelers exist, a related task is to find a set of the labelers which are \emph{good} but not perfect.
We show that we can identify the set of all good labelers, when at least the majority of labelers are good.

\subsection{Overview of Results}
We study various versions of the model described above. In the most basic setting we assume that a large percentage, say $70\%$ of the labelers are \emph{perfect}, i.e., they always label according to the target function $f^*$. The remaining $30\%$ of the labelers could behave arbitrarily and we make no assumptions on them.
Since the perfect labelers are in strong majority, a straightforward approach is to label each example with the majority vote over a few randomly chosen labelers, {to produce the correct label on every instance with high probability.
 However, such an approach leads to a query bound of $O(\log\frac m\delta)$ per labeled example, where $m$ is the size of the training set and $\delta$ is the acceptable probability of failure.} In other words, the cost per labeled example is $O(\log\frac m\delta)$ and scales with the size of the data set.  Another easy approach is to pick a few labelers at random and ask them to label all the examples. Here, the cost per labeled example is a constant but the approach is infeasible in a crowdsourcing environment since it requires a single or a constant number of labelers to label the entire data set.
{Yet another approach is to label each example with the majority vote of $O(\log\frac 1\epsilon)$ labelers. While the labeled sample set created in this way only has error of $\epsilon$, it is still unsuitable for being used with PAC learning algorithms as they are not robust to even small amounts of noise, if the noise is heterogeneous. So, the computational challenges still persist.}
Nevertheless, we introduce an algorithm that performs \emph{efficient} learning with \emph{$O(1)$ cost per labeled example} and \emph{$O(1)$ load  per labeler.}

\medskip

\noindent \textbf{Theorem~\ref{thm:main-manyPerfect-informal}~(Informal)}~\emph{
Let $\F$ be a hypothesis class that can be PAC learned in polynomial time to $\epsilon$ error with probability $1-\delta$ using    $m_{\epsilon, \delta}$ samples. Then $\F$ can be learned in polynomial time using $O(m_{\epsilon, \delta})$ samples in a crowdsourced setting with $O(1)$ cost per labeled example, provided a  $\frac 12 + \Theta(1)$ fraction of the labelers are perfect. Furthermore, every labeler is asked to label only $1$ example.
}
\medskip

Notice that the above theorem immediately implies that each example is queried only $O(1)$ times on average as opposed to the data size dependent $O(\log(\frac m \delta))$ cost incurred by the naive majority vote style procedures.
We next extend our result to the setting where the fraction of perfect labelers is significant but might be less than $\frac 1 2$, say $0.4$. Here we again show that $\F$ can be efficiently PAC learned using $O(m_{\epsilon, \delta})$ queries provided we have access to an ``expert'' that can correctly label  a constant number of examples. We call such queries that are made to an expert {\em golden queries}. 
When the fraction of perfect labelers is close to $\frac 12$, say $0.4$, we show that {\em just one} golden query is enough to learn.
More generally, when the fraction of the perfect labelers is some $\alpha$, we show that $O(1 / \alpha)$ golden queries is sufficient to learn a classifier efficiently. We describe our results in terms of  $\alpha$, but we are particularly interested in regimes where $\alpha = \Theta(1)$.

\medskip

\noindent \textbf{Theorem~\ref{thm:main-halfPerfect-informal}~(Informal)}~\emph{
Let $\F$ be a hypothesis class that can be PAC learned in polynomial time to $\epsilon$ error with probability $1-\delta$ using $m_{\epsilon, \delta}$ samples. Then $\F$ can be learned in polynomial time using $O(m_{\epsilon, \delta})$ samples in a crowdsourced setting with $O(\frac 1 {\alpha})$ cost per labeled example, provided more than an $\alpha$ fraction of the labelers are perfect for some constant $\alpha > 0$. Furthermore, every labeler is asked to label only $O(\frac 1 {\alpha})$ examples and the algorithm uses at most $\frac 2 {\alpha}$ golden queries.
}
\medskip

The above two theorems highlight the importance of incorporating the structure of the crowd in algorithm design. Being oblivious to the labelers will result in noise models that are notoriously hard. For instance, if one were to assume that each example is labeled by a single random labeler drawn from the crowd, one would recover the {\em Malicious Misclassification Noise} of~\cite{rivest1994formal}.
Getting computationally efficient learning algorithms even for very simple hypothesis classes
has been a long standing open problem in this space.
Our results highlight  that by incorporating the structure of the crowd, one can efficiently learn \emph{any hypothesis class} with a small overhead.

Finally, we study the scenario when none of the labelers are perfect. Here we assume that the majority of the labelers are ``good'', that is they provide labels according to functions that are all $\epsilon$-close to the target function. In this scenario generating a hypothesis of low error is as hard as agnostic learning\footnote{This can happen for instance when all the labelers label according to a single function $f$ that is $\epsilon$-far from $f^*$.}. Nonetheless, we show that one can detect all of the good labelers using expected $O(\frac 1 {\epsilon}\log(n))$ queries per labeler, where $n$ is the target number of labelers desired in the pool. 

\medskip
\noindent \textbf{Theorem~\ref{thm:main-noPerfect-informal}~(Informal)}~\emph{
Assume we have a target set of $n$ labelers that are partitioned into two sets, {\em good} and {\em bad}. 
Furthermore, assume that there are at least $\frac n 2$ good labelers who  always provide labels according to functions that are $\epsilon$-close to a target function $f^*$. The set of bad labelers always provide labels according to functions that are at least $4 \epsilon$ away from the target. Then there is a polynomial time algorithm that identifies, with probability at least $1-\delta$, all the good labelers and none of the bad labelers using expected $O(\frac 1 {\epsilon} \log (\frac n \delta))$ queries per labeler.}
\medskip

\subsection{Related Work}
Crowdsourcing has received significant attention in the machine learning community. As mentioned in the introduction, crowdsourcing platforms require one to address several questions that are not present in traditional modes of learning.

The work of~\cite{dekel2009vox} shows how to use crowdsourcing to reduce the noise in a training set before feeding it to a learning algorithm. Our results answer an open question in their work by showing that performing data labeling and learning in tandem can lead to significant benefits.

A large body of work in crowdsourcing  has focused on the problem of {\em task assignment}. Here, workers arrive in an online fashion and a requester has to choose to assign specific tasks to specific workers. Additionally, workers might have different abilities and might charge differently for the same task. The goal from the requester's point of view is to finish multiple tasks within a given budget while maintaining a certain minimum quality~\citep{vaughan2013adaptive, tran2014efficient}. There is also significant work on {\em dynamic procurement} where the focus is on assigning prices to the given tasks so as to provide incentive to the crowd to perform as many of them as possible within a given budget~\citep{BKS12, BKS13, singla2013truthful}. Unlike our setting, the goal in these works is not to obtain a generalization guarantee or learn a function, but  rather to complete as many tasks as possible within the budget.

The work of~\cite{karger2011iterative, karger2014budget} also studies the problem of task assignment
in offline and online settings. 
In the offline setting, the authors provide an algorithm based on
belief propagation that infers the correct answers for each task by pooling together the answers from each worker. They show that their approach performs better than simply taking majority votes. Unlike our setting, their goal is to get an approximately correct set of answers for the given data set and not to generalize from the answers. Furthermore, their model assumes that each labeler makes an error at random independently with a certain probability. We, on the other hand, make no assumptions on the nature of the {\em bad} labelers.

Another related model is the recent work of~\cite{steinhardt2016avoiding}. Here the authors look at the problem of extracting top rated items by a group of labelers among whom a constant fraction are consistent with the {\em true} ratings of the items. The authors use ideas from matrix completion to design an algorithm that can recover the top rated items with an $\epsilon$ fraction of the noise provided every labeler rates $\sim \frac 1 {\epsilon^4}$ items and one has access to $\sim \frac 1 {\epsilon^2}$ ratings from a trusted  expert. Their model is incomparable to ours since their goal is to recover the top rated items and not to learn a hypothesis that generalizes to a test set. 

Our results also shed insights into the notorious problem of PAC learning with noise. Despite decades of research into PAC learning, noise tolerant polynomial time learning algorithms remain elusive. 
There has been substantial work on PAC learning under realistic noise models such as the Massart noise or the Tsybakov noise models~\citep{bbl05}. However, computationally efficient algorithms for such models are known in very restricted cases~\citep{ABHU15,awasthi2016learning}. In contrast, we show that by using the structure of the crowd, one can indeed design polynomial time PAC learning algorithms even when the noise is of the type mentioned above. 

More generally, interactive models of learning have been studied in the machine learning community~\citep{Cohn94, Dasgupta05, BBL06, Kol10, hanneke:11,zhang2015active,yan2016active}.
We describe some of these works in Appendix~\ref{app:related}.

%% file: model.tex
\section{Model and Notations}
 
Let $\X$ be an instance space and $\Y = \{+1, -1\}$ be the set of possible labels.
A \emph{hypothesis} is a function $f: \X \rightarrow \Y$ that maps an instance $x\in \X$ to its classification $y$. 
We consider the \emph{realizable} setting where there is a distribution  over $\X \times \Y$ and a true target function in  hypothesis class $\F$. 
More formally, we consider a distribution $\D$ over $\X\times \Y$ and an unknown hypothesis $f^*\in \F$, where $\err_D(f^*) = 0$.
We denote the marginal of $\D$ over $\X$ by $\DX$. 
The \emph{error} of a hypothesis $f$ with respect to distribution $\D$ is defined as 
$\err_\D(f) = \Pr_{(x,f^*(x))\sim \D} [f(x) \neq f^*(x)]$.

In order to achieve our goal of learning $f^*$ well with respect to distribution $D$, we consider having access to a large pool of labelers, some of whom label according to $f^*$ and some who do not.
Formally, labeler $i$ is defined by its corresponding classification function $\ell_i: \X \rightarrow \Y$.
We say that $\ell_i$ is \emph{perfect} if $\err_D(\ell_i) = 0$.
We consider a distribution $P$ that is uniform over all labelers and let $\alpha = \Pr_{i\sim P}[ \err_D(\ell_i) = 0]$  be the fraction of perfect labelers.
We allow an algorithm to query labelers on  instances drawn from $\DX$.
Our goal is to design learning algorithms that efficiently learn a low error classifier while maintaining a small overhead in the number of labels.
We compare the computational and statistical aspects of our algorithms to their PAC counterparts in the realizable setting. 

In the traditional PAC setting with a realizable distribution,  $m_{\epsilon, \delta}$ denotes the number of samples needed for  learning $\F$.
That is,  $m_{\epsilon, \delta}$ is  the total number of labeled samples drawn from the realizable distribution $\D$ needed to output a classifier $f$ that has $\err_\D(f) \leq \epsilon$, with probability $1-\delta$.
We know from the VC theory~\citep{AB99}, that for a hypothesis class $\F$ with VC-dimension $d$ and no additional assumptions on $\F$, 
$m_{\epsilon, \delta} \in O\left( \epsilon^{-1} \left( d \ln\left( \frac 1\epsilon \right) + \ln\left( \frac 1\delta\right) \right)  \right).$
Furthermore, we assume that efficient algorithms for the realizable setting exist. 
That is, we consider an oracle  $\O_\F$
that for a set of labeled instances $S$, returns a function $f\in \F$ that is consistent with the labels in $S$, if one such function exists, and outputs ``None'' otherwise.

Given an algorithm in the noisy crowd-sourcing setting,
we define the \emph{average cost per labeled example} of the algorithm, denoted by $\query$, to be the ratio of the number of label queries made by the algorithm to the number of labeled examples needed in the traditional realizable PAC model, $m_{\epsilon, \delta}$.
The \emph{load} of an algorithm,  denoted by $\lambda$, is the \emph{maximum number of label queries that have to be answered by an individual labeler.} In other words, $\lambda$ is the maximum number of labels queried from one labeler, when $P$ has an infinitely large support.
\footnote{The concepts of \emph{total number of queries} and  \emph{load} may be seen as analogous to \emph{work} and \emph{depth} in parallel algorithms, where \emph{work} is the total number of operations performed by an algorithm and \emph{depth} is the maximum number of operations that one processor has to perform in a system with infinitely many processors.}
When the number of labelers is fixed, such as in Section~\ref{sec:detect}, we define the
\emph{load} to simply be the number of queries answered by a single labeler.
Moreover, we allow an algorithm to directly query the target hypothesis $f^*$ on a few, e.g., $O(1)$,  instances drawn from  $\DX$. 
We call these ``golden queries'' and denote their  total number by $\golden$.

Given a set of labelers $L$ and an instance $x\in \X$, we define $\maj_L(x)$ to be the label assigned to $x$ by the majority of labelers in $L$. Moreover, we denote by $\smaj_L(x)$ the fraction of the labelers in $L$  that agree with the label $\maj_L(x)$. Given a set of classifiers $H$, we denote by $\Maj(H)$ the classifier that for each $x$ returns prediction $\maj_H(x)$.
Given a distribution $P$ over labelers and a set of labeled examples $S$, we denote by $P_{\mid S}$
the distribution $P$ conditioned on labelers that agree with labeled samples $(x,y) \in S$. 
We consider $S$ to be small, typically of size $O(1)$.
Note that we can draw a labeler from  $P_{\mid S}$ by first drawing a labeler according to $P$ and querying it on all the labeled instances in $S$. Therefore, when $P$ has infinitely large support, the load of an algorithm is the maximum size of $S$ that $P$ is ever conditioned on.

%% file: results.tex
%
%
%
%

%% file: warmup.tex
\section{A Baseline Algorithm and a Road-map for Improvement}  \label{sec:roadmap}

In this section, we briefly describe a simple algorithm and the approach we use to improve over it.
Consider a very simple baseline algorithm for the case of $\alpha > \frac 12$:
\begin{quote}
\textbf{\textsc{Baseline}}:~ Draw a sample of size $m =m_{\epsilon, \delta}$ from $\DX$ and label each $x\in S$ by $\maj_L(x)$, where $L \sim P^k$ for $k= O\left((\alpha - 0.5)^{-2} \ln\left(\frac{m}{\delta} \right) \right)$ is a set of randomly drawn labelers. Return classifier  $\O_\F(S)$.
\end{quote}
That is, the baseline algorithm queries enough labelers on each sample such that with probability $1- \delta$ all the labels are correct. Then, it learns a classifier using this labeled set.
It is clear that the performance of \textsc{Baseline} is far from being desirable. 
First, this approach takes $\log(m/\delta)$ more labels than it requires samples, leading to
an average cost per labeled example that increases with the size of the sample set.
Moreover, when perfect labelers form a small majority of the labelers, i.e.,   $\alpha = \frac 12 + o(1)$, the number of labels needed to correctly label an instance increases drastically.
Perhaps even more troubling is that if the perfect labelers are in minority, i.e., $\alpha < \frac 12$, $S$ may be mislabeled and $\O_\F(S)$ may return a classifier that has large error, or no classifier at all.
In this work, we improve over \textsc{Baseline} in both aspects.

In Section~\ref{sec:boost}, we 
improve the $\log(m/\delta)$ average cost per labeled example by interleaving the two processes responsible for learning a classifier and querying labels.
In particular, \textsc{Baseline} first finds \emph{high quality labels}, i.e., labels that are correct with high probability, and then learns a classifier that is consistent with those labeled samples.
However, interleaving the process of learning and acquiring high quality labels can make both processes more efficient.
At a high level, for a given classifier $h$ that has a larger than desirable error, one may be able to find regions where $h$ performs particularly poorly. That is, the classifications provided by $h$ may differ from the correct label of the  instances. In turn, 
by focusing our effort for getting high quality labels on these regions we can output a correctly labeled sample set using less label queries overall. 
These additional correctly labeled instances from regions where $h$ performs poorly can help us improve the error rate of $h$ in return. In Section~\ref{sec:boost}, we introduce an
algorithm that draws on ideas from boosting and a probabilistic filtering approach that we develop in this work to facilitate interactions between learning and querying.

In Section~\ref{sec:boost-alpha}, we remove the dependence of label complexity on $(\alpha - 0.5)^{-2}$ using $O(1/\alpha)$ golden queries. At a high level, instances where only a small majority of labelers agree are difficult to label using queries asked from labelers. But, these instances are great test cases that help us identify a large fraction of imperfect labelers. That is, we can first ask a golden query on one such instance to get its correct label and from then on only consider labelers that got this label correctly. In other words, we first test the labelers on one or very few tests questions, if they pass the tests, then  we ask them real label queries for the remainder of the algorithm, if not, we never consider them again.

%% file: boost.tex
\section{An Interleaving Algorithm}  \label{sec:boost}

In this section, we improve over the average cost per labeled example of the \textsc{Baseline} algorithm, by interleaving the process of learning and acquiring high quality labels. 
Our Algorithm~\ref{alg:boost} facilitates the interactions between the learning process and the querying process using ideas from classical PAC learning and adaptive techniques we develop in this work. 
For ease of presentation, we first consider the case where $\alpha = \frac 12 + \Theta(1)$, say $\alpha \geq 0.7$, and introduce an algorithm and techniques that work in this regime. In Section~\ref{sec:boost-alpha}, we  show how our algorithm can be modified to work with any value of $\alpha$.
For convenience, we assume in the analysis below that distribution $\D$ is over a discrete space. This is in fact without loss of generality, since using uniform convergence one can instead work with the uniform distribution over an unlabeled sample multiset of size $O(\frac {d}{\epsilon^2})$ drawn  from $\D_{|\X}$.

Here, we provide an overview of the techniques and ideas used in this algorithm.
\paragraph{Boosting:}
In general, boosting algorithms~\citep{schapire1990strength,freund1990boosting,freund1995desicion} provide a mechanism for producing a classifier of error $\epsilon$ using learning algorithms that are only capable of producing classifiers with considerably larger error rates, typically of error $p = \frac 12-\gamma$ for small $\gamma$. 
In particular,  early work of \cite{schapire1990strength} in this space shows  how one can combine $3$ classifiers of error $p$ to get a classifier of error $O(p^2)$, for any $p>0$.

\begin{theorem}[\cite{schapire1990strength}] \label{thm:schapire}
For any $p>0$ and distribution $\D$, consider three classifiers:
1) classifier $h_1$ such that $\err_\D(h_1)\leq p$;
2) classifier $h_2$ such that $\err_{D_2}(h_2) \leq p$, where $\D_2 = \frac 12 D_C + \frac 12 D_I$ for distributions $\D_C$ and $D_I$ that denote  distribution $\D$  conditioned on $\{x\mid h_1(x)= f^*(x)\}$ and $\{x\mid h_1(x)\neq f^*(x)\}$, respectively;
3) classifier $h_3$ such that $\err_{D_3}(h_3) \leq p$, where $\D_3$ is $\D$  conditioned on $\{x\mid h_1(x) \neq h_2(x)\}$.
Then, $\err_D(\Maj(h_1, h_2, h_3)) \leq 3p^2 - 2p^3$.
\end{theorem}

As opposed to the main motivation for boosting where the learner only has access to a learning algorithm of error $p = \frac 12  - \gamma$, in our setting we can learn a classifier to \emph{any desired error rate $p$} as long as we have a sample set of $m_{p, \delta}$ correctly labeled instances. The larger the error rate $p$, the smaller the total number of label queries needed for producing a correctly labeled set of the appropriate size. 
We use this idea in Algorithm~\ref{alg:boost}. In particular, we learn
classifiers of error $O(\sqrt{\epsilon})$ using sample sets of size $O(m_{\sqrt \epsilon, \delta})$ that are labeled by majority vote of 
$O(\log(m_{\sqrt \epsilon, \delta}))$ labelers, using fewer label queries overall than \textsc{Baseline}.

\paragraph{Probabilistic Filtering:}
Given classifier $h_1$, the second step of the classical boosting algorithm requires distribution $\D$ to be reweighed based  on the correctness of $h_1$. This step can be done by a \emph{filtering} process as follows: Take a large set of labeled samples from $\D$ and divide them to two sets depending on whether or not the instances are mislabeled by $h_1$. Distribution $\D_2$, in which instances mislabeled by $h_1$ make up half of the weight, can be simulated by picking each set with probability $\frac 12$ and taking an instance from that set uniformly at random. 
To implement \emph{filtering} in our setting, however, we would need to first get high quality labels for the set of instances used for simulating $\D_2$. Furthermore, this sample set is typically large, since at least  $\frac 1p m_{p, \delta}$ random samples from $\D$ are needed to simulate $\D_2$ that has half of its weight on the points that $h_1$ mislabels (which is a $p$ fraction of the total points).
In our case where $p = O(\sqrt{\epsilon})$, getting high quality labels for such a large sample set requires $O\left(m_{\epsilon, \delta} \ln\left( \frac{m_{\epsilon, \delta}}{\delta}\right) \right)$ label queries, which is as large as the total number of labels queried by \textsc{Baseline}.

\begin{algorithm}
\SetAlgoNoLine
Let $S_I = \emptyset$ and $N = \log\left( \frac{1}{\epsilon}  \right)$.\\
\For{$x\in S$}{
	\For{$t = 1, \dots, N$}{
		Draw a random labeler  $i\sim P$ and let $y_t = \ell_i(x)$\\
		\textbf{If} $t$ is odd and $\maj(y_{1:t}) = h(x)$, \textbf{then} break.   \label{item:break} 
	}
Let $S_I = S_I \cup \{ x\}$.     \hfill  // Reaches this step when for all $t$, $\maj(y_{1:t}) \neq h(x)$}
\Return{$S_I$}
\caption{\textsc{Filter}$(S, h)$}
\label{alg:filter}
\end{algorithm}

In this work, we introduce a \emph{probabilistic filtering} approach, called \textsc{Filter},  that only requires $O\left( m_{\epsilon, \delta} \right)$ label queries, i.e., $O(1)$ cost per labeled example.
Given classifier $h_1$ and an unlabeled sample set $S$, $\textsc{Filter}(S, h_1)$ returns a set $S_I\subseteq S$ such that \emph{for any $x\in S$ that is mislabeled by $h_1$, $x\in S_I$ with probability at least $\Theta(1)$. Moreover, any $x$ that is correctly labeled by $h_1$ is most likely not included in $S_I$.}
This procedure is described in detail in Algorithm~\ref{alg:filter}. Here, we provide a brief description of its working:~
For any $x \in S$, \textsc{Filter} queries one labeler at a time, drawn at random, until the majority of the labels it has acquired so far agree with $h_1(x)$, at which point \textsc{Filter} removes $x$ from consideration. On the other hand, if the majority of the labels never agree with $h_1(x)$, \textsc{Filter} adds $x$ to the output set $S_I$.
Consider $x\in S$ that is correctly labeled by $h$. Since each additional label agrees with $h_1(x) = f^*(x)$ with probability $\geq 0.7$, with high probability the majority of the labels on $x$ will agree with $f^*(x)$ at some point, in which case \textsc{Filter} stops asking for more queries and removes $x$. As we show in Lemma~\ref{lem:h_2-label} this happens within $O(1)$ queries most of the time.
On the other hand, for $x$ that is mislabeled by $h$, a labeler agrees with $h_1(x)$ with probability $\leq 0.3$. Clearly, for one set of random labelers ---one snapshot of the labels queried by \textsc{Filter}--- the majority label agrees with $h_1(x)$ with a very small probability. As we show in Lemma~\ref{lem:filter}, even when considering the progression of all labels queried
by \textsc{Filter} throughout the process, with probability $\Theta(1)$ the majority label never agrees with $h_1(x)$. Therefore, $x$ is added to $S_I$ with probability $\Theta(1)$. 

\paragraph{Super-sampling:} 
Another key technique we use in this work is \emph{super-sampling}. In short, this  means that as long as we have the correct label of the sampled points and we are in the realizable setting, more samples never hurt the algorithm. Although this  seems trivial at first, it does play an important role in our approach.
In particular, our probabilistic filtering procedure does not necessarily simulate $\D_2$ but a distribution $\D'$, such that $\Theta(1)d_2(x) \leq d'(x)$ for all $x$, where $d_2$ and $d'$ are the densities of $D_2$ and $D'$, respectively. At a high level, sampling $\Theta(m)$ instances from $\D'$ simulates a super-sampling process that samples $m$ instances from $\D_2$ and then adds in some arbitrary instances. This is formally stated below and is proved in Appendix~\ref{app:super-sample}.

\begin{lemma}\label{lem:super-sample}
Given a hypothesis class $\F$ consider any two discrete distributions $\D$ and $\D'$ such that for all $x$, $d'(x) \geq c\cdot d(x)$ for an absolute constant $c>0$, and both distributions are labeled according to $f^*\in \F$.
There exists a constant $c'>1$ such that for any  $\epsilon$ and $\delta$, with probability $1-\delta$ over a labeled  sample set $S$ of size $c' m_{\epsilon, \delta}$ drawn from  $\D'$, $\O_\F(S)$ has error of at most $\epsilon$ with respect to distribution $\D$.
\end{lemma}

With these techniques at hand, we present Algorithm~\ref{alg:boost}. At a high level, the algorithm proceeds in three phases, one for each classifier used by Theorem~\ref{thm:schapire}.
In Phase 1, the  algorithm learns $h_1$ such that $\err_{\D}(h_1) \leq \frac 12\sqrt{\epsilon}$.
In Phase 2, the algorithm first filters a set of size $O(m_{\epsilon, \delta})$ into the set $S_I$ and takes an additional
set $S_C$ of $\Theta(m_{\sqrt \epsilon, \delta})$ samples. Then, it queries $O(\log(\frac{m_{\epsilon, \delta}}{\delta}))$ labelers on each instance in $S_I$ and $S_C$ to get their correct labels with high probability. Next, it partitions these instances to two different sets based on whether or not $h_1$ made a mistake on them. It then learns $h_2$ on a sample set $\bar W$ that is drawn by weighting these two sets equally. As we show in Lemma~\ref{lem:h_2-err}, $\err_{\D_2}(h_2) \leq \frac 12 \sqrt\epsilon$.
In phase 3, the algorithm learns $h_3$ on a sample set $S_3$  drawn from $\DX$ conditioned on $h_1$ and $h_2$ disagreeing. Finally, the algorithm returns $\Maj(h_1, h_2, h_3)$.

\begin{algorithm}[t]
\SetAlgoNoLine
\KwIn{Given a distribution $\DX$, a class of hypotheses $\F$, parameters $\epsilon$ and $\delta$.}
\vspace*{2pt}

\noindent{\textbf{Phase 1}}:

\Indp Let $\bar{S_1} = \textsc{Correct-Label}(S_1, \delta/6)$,  for a set of sample $S_1$ of size $2 m_{\sqrt{\epsilon}, \delta/6}$ from $\DX$.

Let $h_1  = \O_\F(\bar{S_1})$.

\Indm \noindent{\textbf{Phase 2}}:

\Indp Let $S_I = \textsc{Filter} (S_2, h_1)$, for a set of samples $S_2$ of size $\Theta(m_{\epsilon, \delta})$ drawn from $\DX$.

Let $S_C$ be a sample set of size $\Theta(m_{\sqrt{\epsilon}, \delta})$ drawn from $\DX$.

Let $\bar{S_{All}} = \textsc{Correct-Label}(S_I \cup S_C, \delta/6)$.

Let $\bar{W_I} = \{ (x, y) \in \bar{S_{All}} \mid y \neq h_1(x)\}$ and Let $\bar{W_C} = \bar{S_{All}} \setminus \bar{W_I}$.

Draw a sample set $\bar W$ of size $\Theta(m_{\sqrt{\epsilon}, \delta})$ from a distribution that equally  weights $\bar{W_I}$ and $\bar{W_C}$.

Let $h_2 = \O_\F(\bar{W})$.

\Indm \noindent{\textbf{Phase 3}}:

\Indp Let $\bar{S_3} = \textsc{Correct-Label}(S_3, \delta/6)$, for a sample set $S_3$ of size  $2m_{\sqrt{\epsilon}, \delta/6}$ drawn from $\DX$ conditioned on $h_1(x) \neq h_2(x)$.

Let $h_3 = \O_\F(\bar{S_3})$.

\Indm \Return{$\maj(h_1, h_2, h_3)$.}
\vspace*{4pt}

\hrule \vspace*{2pt}
\textbf{\textsc{Correct-Label}$(S, \delta)$:}
\vspace*{2pt}\hrule 

\For{$x\in S$}{ 
	Let $L\sim P^k$ for a set of $k= O(\log(\frac{|S|}{\delta}))$ labelers drawn from $P$ and
		$\bar S \gets  \bar S \cup \{ (x, \maj_L(x)) \}$.	
}
\Return{$\bar S$}.
\caption{\textsc{Interleaving: Boosting By Probabilistic Filtering for $\alpha = \frac 12 + \Theta(1)$}}
\label{alg:boost}
\end{algorithm}

\begin{theorem}[$\boldsymbol{\mathbf{\alpha = \frac 12 + \Theta(1)}}$ case]
\label{thm:main-manyPerfect-informal}
Algorithm~\ref{alg:boost} uses oracle $\O_\F$, runs in time $\poly(d, \frac 1\epsilon, \ln(\frac 1\delta))$ and with probability $1-\delta$ returns $f\in \F$ with $\err_{\D}(f) \leq \epsilon$, using $\query = O\left(  \sqrt{\epsilon} \log\left(\frac{m_{\sqrt{\epsilon}, \delta}}{\delta}  \right) + 1\right)$ cost per labeled example, $\golden = 0$ golden queries, and  $\load = 1$ load.
Note that when $\frac{1}{\sqrt \epsilon} \geq  \log\left(\frac{m_{\sqrt{\epsilon}, \delta}}{\delta}  \right)$, the above cost per labeled sample is $O(1)$.
\end{theorem}

We start our analysis of Algorithm~\ref{alg:boost} by stating that $\textsc{Correct-Label}(S, \delta)$ labels $S$ correctly, with probability $1-\delta$. This is direct application of the Hoeffding bound and its proof is omitted.
\begin{lemma}
For any unlabeled sample set  $S$, $\delta >0$, and $\bar{S} = \textsc{Correct-Label}(S, \delta)$, 
with probability $1-\delta$, for all $(x,y) \in \bar S$, $ y= f^*(x)$.
\end{lemma}

Note that as a direct consequence of the above lemma, Phase 1 of Algorithm~\ref{alg:boost} achieves error of $O(\sqrt \epsilon)$.
\begin{lemma} \label{lem:h_1-err}
In Algorithm~\ref{alg:boost},  with probability $1 - \frac \delta 3$, $\err_D(h_1) \leq \frac 12 \sqrt{\epsilon}$.
\end{lemma}

Next, we prove that \textsc{Filter}  removes instances that are correctly labeled by $h_1$ with good probability and retains instances that are mislabeled by $h_1$ with at least a constant probability.

\begin{lemma} \label{lem:filter}
Given any sample set $S$ and classifier $h$, for every  $x\in S$
\begin{enumerate}[topsep=0pt,itemsep=0pt,partopsep=0pt,parsep=0pt]
\item  If $h(x) = f^*(x)$, then $x \in \textsc{Filter}(S, h)$ with probability $<\sqrt{\epsilon}$. 
\item  If $h(x) \neq f^*(x)$, then $x \in \textsc{Filter}(S, h)$ with probability $\geq 0.5$. 
\end{enumerate}
\end{lemma}
\begin{proof}
For the first claim, note that $x \in S_I$  only if $\maj(y_{1:t}) \neq h(x)$ for all $t\leq N$. 
Consider $t=N$ time step.
Since each random query agrees with $f^*(x) = h(x)$ with probability $\geq 0.7$ independently, majority of $N = O(\log(1/\sqrt{\epsilon}))$ labels are correct with probability at least $1- \sqrt{\epsilon}$. Therefore, the probability that the majority label disagrees with $h(x) = f^*(x)$ at every time step is at most $\sqrt{\epsilon}$.

In the second claim, we are interested in the probability that there exists some $t\leq N$, for which  $\maj(y_{1:t}) = h(x) \neq f^*(x)$.
This is the same as the probability of return in biased random walks, also called the probability of ruin in  gambling~\citep{feller2008introduction}, where we are given a random walk that takes a step to the right with probability $\geq 0.7$ and takes a step to the left with the remaining probability and we are interested in the probability that this walk  ever crosses  the origin to the left while taking $N$ or even infinitely many steps. Using the probability of return for biased random walks (see Theorem~\ref{thm:ruin}), the probability that $\maj(y_{1:t}) \neq f^*(x)$ ever is   at most
$\left( 1 - \left( \frac{0.7}{1- 0.7} \right)^{N} \right)  /  \left( 1 - \left( \frac{0.7}{1-0.7} \right)^{N+1} \right) <  \frac{3}{7}.
$
Therefore, for each $x$ such that  $h(x) \neq f^*(x)$, $x\in S_I$ with probability at least $4/7$.
\end{proof}

In the remainder of the proof, for ease of exposition we assume that not only $\err_{\D}(h_1) \leq \frac 12 \sqrt{\epsilon}$ as per Lemma~\ref{lem:h_1-err}, but in fact $\err_{\D}(h_1) = \frac 12 \sqrt{\epsilon}$. This assumption is not needed for the correctness of the results but it helps simplify the notation and analysis.
As a direct consequence of Lemma~\ref{lem:filter} and application of the Chernoff bound, we deduce that with high probability $\bar{W}_I$, $\bar{W}_C$, and $S_I$ all have size $\Theta(m_{\sqrt \epsilon, \delta})$. 
The next lemma, whose proof appears in Appendix~\ref{app:size}, formalizes this claim.
\begin{lemma}\label{lem:size}
With probability $1-\exp(- \Omega(m_{\sqrt\epsilon, \delta}))$,  $\bar{W}_I$, $\bar{W}_C$, and $S_I$ all have size  $\Theta( m_{\sqrt\epsilon, \delta})$. 
\end{lemma}

The next lemma combines the probabilistic filtering and super-sampling techniques to show that $h_2$  has the desired error  $O(\sqrt \epsilon)$ on $\D_2$.

\begin{lemma} \label{lem:h_2-err}
Let $\D_C$ and $\D_I$ denote  distribution $\D$ when it is conditioned on $\{x\mid h_1(x)= f^*(x)\}$ and $\{x\mid h_1(x)\neq f^*(x)\}$, respectively, and let $\D_2 = \frac 12 \D_I + \frac 12 \D_C$. With probability $1- 2\delta/3$, $\err_{\D_2}(h_2) \leq \frac12 \sqrt{\epsilon}$.
\end{lemma}

\begin{proof}
Consider distribution $\D'$ that has equal probability on the distributions induced by $\bar{W_I}$ and $\bar{W_C}$ and let $d'(x)$ denote the density of point $x$ in this distribution.
Relying on our \emph{super-sampling} technique (see Lemma~\ref{lem:super-sample}), it is sufficient to show that for any $x$, $d'(x) = \Theta(d_2(x))$.

For ease of presentation, we assume that Lemma~\ref{lem:h_1-err} holds with equality, i.e., $\err_{\D}(h_1)$ is exactly $\frac 12 \sqrt{\epsilon}$ with probability $1 - \delta/3$.
Let $d(x)$, $d_2(x)$, $d_C(x)$, and $d_I(x)$ be the density of instance $x$ in distributions $\D$, $\D_2$, $\D_C$, and $\D_I$, respectively. Note that, for any $x$ such that $h_1(x) = f^*(x)$, we have $d(x) = d_C(x) (1-  \frac 12\sqrt{\epsilon})$. Similarly, for any $x$ such that $h_1(x) \neq f^*(x)$, we have $d(x) = d_I(x)\frac 12 \sqrt{\epsilon}$.
Let $N_C(x)$, $N_I(x)$, $M_C(x)$ and $M_I(x)$  be the number of occurrences of $x$ in the sets $S_C$, $S_I$, $\bar{W_C}$ and $\bar{W_I}$, respectively.  
For any $x$, there are two cases:

\medskip
\noindent{If $h_1(x) = f^*(x)$:} Then, there exist absolute constants $c_1$ and $c_2$ according to Lemma~\ref{lem:size}, such that
\begin{align*}
d'(x)&= \frac 12 \E\left[ \frac{M_C(x)}{|\bar{W_C}|} \right] \geq \frac{\E[M_C(x)]}{ c_1 \cdot  m_{\sqrt \epsilon, \delta} } \geq \frac{\E[N_C(x)]}{c_1 \cdot  m_{\sqrt \epsilon, \delta} } = \frac{| S_C | \cdot d(x) }{c_1 \cdot  m_{\sqrt \epsilon, \delta} } \\
& = \frac{| S_C | \cdot d_C(x) \cdot (1- \frac 12 \sqrt{\epsilon}) }{c_1 \cdot  m_{\sqrt \epsilon, \delta} } \geq 
c_2 d_C(x) = \frac{c_2 d_2(x)}{2},
\end{align*}
where the second and sixth transitions are by the sizes of $\bar{W_C}$ and $|S_C|$ and the third transition is by the fact that if $h(x) = f^*(x)$, $M_C(x) > N_C(x)$.

\noindent{If $h_1(x) \neq f^*(x)$:} Then, there exist absolute constants $c'_1$ and $c'_2$ according to Lemma~\ref{lem:size}, such that
\begin{align*}
d'(x) &= \frac 12 \E\left[ \frac{M_I(x)}{|\bar{W_I}|} \right] \geq \frac{\E[M_I(x)]}{c'_1 \cdot  m_{\sqrt \epsilon, \delta} }
 \geq \frac{\E[N_I(x)]}{c'_1 \cdot  m_{\sqrt \epsilon, \delta} }\geq \frac{ \frac 4 7~  d(x) | S_2|}{c'_1 \cdot  m_{\sqrt \epsilon, \delta} }\\
& = \frac{ \frac 4 7~  d_I(x) \frac 12 \sqrt{\epsilon} \cdot | S_2| }{c'_1 \cdot  m_{\sqrt \epsilon, \delta} } \geq  c'_2 d_I(x) = \frac{c'_2 d_2(x)}{2},
\end{align*}
where the second and sixth transitions are by the sizes of $\bar{W_I}$ and $|S_2|$,
the third transition is by the fact that if $h(x) \neq f^*(x)$, $M_I(x) > N_I(x)$, and the fourth transition holds by part 2 of Lemma~\ref{lem:filter}.

Using the super-sampling guarantees of Lemma~\ref{lem:super-sample}, with probability $1-2\delta/3$, $\err_{\D_2}(h_2) \leq \sqrt{\epsilon} / 2$.
\end{proof}

The next claim shows that the probabilistic filtering step queries a few labels only. At a high level, this is achieved by showing that any instance $x$ for which $h_1(x) = f^*(x)$  contributes only $O(1)$ queries, with high probability. On the other hand, instances that $h_1$ mislabeled may each get $\log(\frac 1\epsilon)$ queries. But, because there are only few such points, the total number of queries these instances require  is a lower order term.

\begin{lemma} \label{lem:h_2-label}
Let $S$ be a sample set drawn from distribution $\D$ and let $h$ be such that $\err_{\D}(h) \leq \sqrt{\epsilon}$. With probability $1 - \exp(- \Omega(|S| \sqrt{\epsilon}))$,  $\textsc{Filter}(S,h)$ makes  $O(|S|)$ label queries.
\end{lemma}
\begin{proof}
Using Chernoff bound, with probability $1 - \exp\left(-|S| \sqrt{\epsilon} \right)$ the total number of points in $S$ where $h$ disagrees with $f^*$ is $O(|S| \sqrt{\epsilon})$. The number of queries spent on these points is at most 
$O \left( |S| \sqrt{\epsilon} \log(1/\epsilon) \right) \leq O(|S|). 
$

Next, we show that for each $x$ such that $h(x) = f^*(x)$, the number of queries taken until a majority of them agree with $h(x)$ is a constant. Let us first show that this is the case in expectation. Let $N_i$ be the expected number of labels  queried until we have  $i$ more correct labels than incorrect ones. Then $N_1 \leq 0.7 (1) + 0.3 (N_2 +1)$, since with probability at least $\alpha \geq 0.7$, we receive one more correct label and stop, and with probability $\leq 0.3$ we get a wrong label in which case we have to get two more correct labels in future. Moreover, $N_2 = 2N_1$, since we have to get one more correct label to move from $N_2$ to $N_1$ and then one more. Solving these, we have that $N_1 \leq 2.5$. Therefore, the expected total number of queries is at most $O(|S|)$. Next, we show that this random variable is also 
well-concentrated. 
Let $L_x$ be a random variable that indicates the total number of queries on $x$ before we have one more correct label than incorrect labels.
Note that $L_x$ is an unbounded random variable, therefore concentration bounds such as Hoeffding or Chernoff do not work here.  Instead, to show that $L_x$ is  well-concentrated, we prove that the Bernstein inequality (see Theorem~\ref{thm:bernstein}) holds.
That is, as we show in Appendix~\ref{app:h_2-label_bernstein}, for any $x$, the Bernstein inequality is statisfied by the fact that for any $i>1$,
$
\E[(L_x - \E[L_x])^i]\leq 50 (i+1)! \, e^{4i}.
$
Therefore, over all instances in $S$,  $\sum_{x\in S } L_x \in O(|S|)$ with probability $1 - \exp(-|S|)$.
\end{proof}

Finally, we have all of the ingredients needed for proving our main theorem.

\begin{proof}[\textbf{Proof of Theorem~\ref{thm:main-manyPerfect-informal}}]
We first discuss the number of label queries Algorithm~\ref{alg:boost} makes.
The total number of labels queried by Phases 1 and 3 is attributed to 
the labels queried by $\textsc{Correct-Label}(S_1, \delta)$ and $\textsc{Correct-Label}(S_3, \delta/6)$, which is 
$O\left(m_{\sqrt{\epsilon}, \delta}  \log(m_{\sqrt{\epsilon}, \delta} / \delta) \right)$.
By Lemma~\ref{lem:size}, 
$| S_I \cup S_C | \leq O(m_{\sqrt\epsilon, \delta})$ almost surely. Therefore,  $\textsc{Correct-Label}( S_I \cup S_C , \delta/6)$ contributes 
$O\left( m_{\sqrt{\epsilon}, \delta}  \log (m_{\sqrt{\epsilon}, \delta} / \delta ) \right)$ labels. Moreover, as we showed in Lemma~\ref{lem:h_2-label}, $\textsc{Filter}(S_2, h_1)$ queries $O(m_{\epsilon, \delta})$ labels, almost surely.
So, the total number of labels queried by Algorithm~\ref{alg:boost} is at most $O\left(  m_{\sqrt{\epsilon}, \delta}  \log\left(\frac{m_{\sqrt{\epsilon}, \delta}}{\delta}  \right) + m_{\epsilon, \delta}\right)$. This leads to $\query = O\left(  \sqrt{\epsilon} \log\left(\frac{m_{\sqrt{\epsilon}, \delta}}{\delta}  \right) + 1\right)$ cost per labeled example.

It remains to show that $\Maj(h_1, h_2, h_3)$ has error $\leq \epsilon$ on $\D$.
Since $\textsc{Correct-Label}(S_1, \delta/6)$ and $\textsc{Correct-Label}(S_3, \delta/6)$ return correctly labeled sets , $\err_{\D}(h_1) \leq \frac 12 \sqrt{\epsilon}$ and $\err_{\D_3}(h_3) \leq \frac 12 \sqrt{\epsilon}$, where $\D_3$ is distribution $\D$ conditioned on $\{x\mid h_1(x) \neq h_2(x)\}$. As we showed in Lemma~\ref{lem:h_2-err}, $\err_{\D_2}(h_2) \leq \frac 12 \sqrt{\epsilon}$ with probability $1-2\delta/3$. Using the boosting technique of \cite{schapire1990strength} described in Theorem~\ref{thm:schapire}, we conclude that 
$\Maj(h_1, h_2, h_3)$ has error $\leq \epsilon$ on $\D$.
\end{proof}

%% file: anyAlpha.tex
\subsection{The General Case of  Any $\alpha$}  \label{sec:boost-alpha}

In this section, we extend Algorithm~\ref{alg:boost} to handle any value of $\alpha$, that does not  necessarily satisfy $\alpha > \frac 12 + \Theta(1)$.
We show that by using $O(\frac 1 {\alpha})$ golden queries, it is possible to efficiently learn any function class with a small overhead.

There are two key challenges that one needs to overcome when $\alpha < \frac 12 + o(1)$. First, we can no longer assume that by taking the  majority vote over a few random labelers we get the correct label of an instance. Therefore, $\textsc{Correct-Label}(S, \delta)$ may return a highly noisy labeled sample set. This is problematic, since efficiently learning $h_1, h_2$, and $h_3$ using oracle $\O_\F$ crucially depends on the correctness of the input labeled set.
Second, $\textsc{Filter}(S,  h_1)$ no longer ``filters'' the instances correctly based on the classification error of $h_1$.
In particular, \textsc{Filter} may retain a constant fraction of instances where $h_1$ is in fact correct, and it may throw out instances where $h_1$ was incorrect with high probability. Therefore, the per-instance guarantees of Lemma~\ref{lem:filter} fall apart, immediately.

We overcome both of these challenges by using two key ideas outlined below.

\noindent \textbf{Pruning: }
As we alluded to in Section~\ref{sec:roadmap}, instances where only a small majority of labelers are in agreement are great for identifying and pruning away a noticeable fraction of the bad labelers. We call these instances \emph{good test cases}.
In particular, if we ever encounter a good test case $x$, we can ask a golden query $y = f^*(x)$ and from then on only consider the labelers who got this test correctly, i.e.,  $P \gets P_{\mid \{(x,y) \}}$. Note that if we make our golden queries when $\smaj_P(x) \leq 1 - \frac \alpha 2$, at least an $\frac \alpha 2$ fraction of the labelers would be pruned. This can be repeated at most $O(\frac 1 {\alpha})$ times before the number of good labelers form a strong majority, in which case Algorithm~\ref{alg:boost} succeeds.
The natural question is how would we measure $\smaj_P(x)$ using few label queries? Interestingly, $\textsc{Correct-Label}(S, \delta)$ can be modified to detect such good test cases by measuring the empirical agreement rate on a set $L$ of 
$O(\frac 1 {\alpha^2} \log( \frac{|S|} {\delta}))$ labelers.
This is shown in procedure $\textsc{Prune-and-Label}$ as part Algorithm~\ref{alg:anyAlpha}. That is,
if $\smaj_L(x) > 1- \alpha/4$, we take $\maj_L(x)$ to be the label, otherwise we test and prune the labelers, and then restart the procedure.
This ensures that whenever we use a sample set that is labeled by $\textsc{Prune-and-Label}$, we can be certain of the  correctness of the labels. This is stated in the following lemma, and proved in Appendix~\ref{app:test-and-label-proof}.

\begin{lemma}
\label{lem:test-and-label}
For any unlabeled sample set  $S$, $\delta >0$, with probability $1-\delta$, either 
$\textsc{Prune-and-Label}(S, \delta)$ prunes the set of labelers or $\bar{S} = \textsc{Prune-and-Label}(S, \delta)$ is such that for all $(x,y) \in \bar S$, $ y= f^*(x)$.
\end{lemma}

As an immediate result, the first phase of Algorithm~\ref{alg:anyAlpha} succeeds in computing $h_1$, such that $\err_D(h_1)\leq \frac 12  \sqrt{\epsilon}$.
Moreover, every time $\textsc{Prune-and-Label}$ prunes the set of labelers, the total fraction of good labeler among all remaining labelers increase. As we show, after $O(1/\alpha)$ prunings, the set of good labelers is guaranteed to form a large majority, in which case Algorithm~\ref{alg:boost} for the case of $\alpha = \frac 12 + \Theta(1)$ can be used. This is stated in the next lemma and proved in Appendix~\ref{app:depth-golden-proof}.

\begin{lemma}\label{lem:depth-golden}
For any $\delta$, with probability $1- \delta$, the total number of times that Algorithm~\ref{alg:anyAlpha} is restarted as a result of pruning is $O(\frac 1 {\alpha})$.
\end{lemma}

\noindent \textbf{Robust Super-sampling: } 
The filtering step faces a completely different challenge: Any point that is a good test case can be filtered the wrong way.
However, instances where still a strong majority of the labelers agree are not affected by this problem and will be filtered correctly.
Therefore, as a first step we ensure that the total number of good test cases that were not caught before \textsc{Filter} starts is small. For this purpose, we start the algorithm by calling \textsc{Correct-Label} on a sample of size $O(\frac 1\epsilon \log(\frac 1\delta))$, and if no test points were found in this set, then with high probability the total fraction of good test cases in the underlying distribution is at most $\frac \epsilon 2$. 
Since the fraction of good test cases is very small, one can show that except for an $\sqrt{\epsilon}$ fraction, the noisy distribution constructed by the filtering process will, for the purposes of boosting, satisfy the conditions needed for the super-sampling technique. Here, we introduce a robust version of the super-sampling technique to argue that the filtering step will indeed produce $h_2$ of error $O(\sqrt{\epsilon})$.

\begin{lemma}[Robust Super-Sampling Lemma]\label{lem:super-sample-robust}
Given a hypothesis class $\F$ consider any two discrete distributions $\D$ and $\D'$ such that except for an $\epsilon$ fraction of the mass under $\D$, we have that for all $x$, $d'(x) \geq c\cdot d(x)$ for an absolute constant $c>0$ and both distributions are labeled according to $f^*\in \F$.
There exists a constant $c'>1$ such that for any  $\epsilon$ and $\delta$, with probability $1-\delta$ over a labeled  sample set $S$ of size $c' m_{\epsilon, \delta}$ drawn from  $\D'$, $\O_\F(S)$ has error of at most $2\epsilon$ with respect to $\D$.
\end{lemma}

By combining these techniques at every execution of our algorithm we ensure that if a good test case is ever detected we prune a small fraction of the bad labelers and restart the algorithm, and if it is never detected, our algorithm returns a classifier of error $\epsilon$.

\begin{algorithm}[t]
\SetAlgoNoLine
\KwIn{Given a distribution $\DX$ and $P$,  a class of hypothesis $\F$, parameters $\epsilon$, $\delta$, and $\alpha$.}
\vspace*{2pt}

\noindent{\textbf{Phase 0}}:

\Indp If $\alpha > \frac  3 4$, run Algorithm~\ref{alg:boost} and quit. 

Let $\delta' = c\alpha \delta$ for small enough $c>0$ and draw $S_0$ of $O(\frac 1 {\epsilon} \log(\frac{1}{\delta'}))$ examples from the distribution $D$. \\
$\textsc{Prune-and-Label}(S_0, \delta')$.

\Indm \noindent{\textbf{Phase 1}}:

\Indp Let $\bar{S_1} = \textsc{Prune-and-Label}(S_1, \delta')$,  for a set of sample $S_1$ of size $2 m_{\sqrt{\epsilon}, \delta'}$ from $D$.

Let $h_1  = \O_\F(\bar{S_1})$.

\Indm \noindent{\textbf{Phase 2}}:

\Indp Let $S_I = \textsc{Filter} (S_2, h_1)$, for a set of samples $S_2$ of size $\Theta(m_{\epsilon, \delta'})$ drawn from $D$.

Let $S_C$ be a sample set of size $\Theta(m_{\sqrt{\epsilon}, \delta'})$ drawn from $D$.

Let $\bar{S_{All}} = \textsc{Prune-and-Label}(S_I \cup S_C, \delta')$.

Let $\bar{W_I} = \{ (x, y) \in \bar{S_{All}} \mid y \neq h_1(x)\}$ and Let $\bar{W_C} = \bar{S_{All}} \setminus \bar{W_I}$.

Draw a sample set $\bar W$ of size $\Theta(m_{\sqrt{\epsilon}, \delta'})$ from a distribution that equally  weights $\bar{W_I}$ and $\bar{W_C}$.

Let $h_2 = \O_\F(\bar{W})$.

\Indm \noindent{\textbf{Phase 3}}:

\Indp Let $\bar{S_3} = \textsc{Prune-and-Label}(S_3, \delta')$, for a sample set $S_3$ of size  $2m_{\sqrt{\epsilon}, \delta'}$ drawn from $D$ conditioned on $h_1(x) \neq h_2(x)$.

Let $h_3 = \O_\F(\bar{S_3})$.

\Indm \Return{$\maj(h_1, h_2, h_3)$.}
\vspace*{4pt}

\hrule \vspace*{2pt}
\textbf{\textsc{Prune-and-Label}$(S, \delta)$:}
\vspace*{2pt}\hrule 

\For{$x\in S$}{ 
	Let $L\sim P^k$ for a set of $k= O(\frac 1 {\alpha^2} \log(\frac{|S|}{\delta}))$ labelers drawn from $P$.\\
    \uIf{$\smaj_{L}(x)\leq 1- \frac {\alpha} 4$}{	
	     Get a golden query $y^* = f^*(x)$,\\
	    Restart Algorithm~\ref{alg:anyAlpha} with distribution $P \gets P_{\mid \{ (x, y^*) \}}$ and $\alpha \gets \frac{\alpha}{1 - \frac \alpha 8 }$.
		}
		\Else{		$\bar S \gets  \bar S \cup \{ (x, \maj_L(x)) \}$.	
		}
}

\Return{$\bar S$}.
\caption{\textsc{Boosting By Probabilistic Filtering for any $\alpha$}}
\label{alg:anyAlpha}
\end{algorithm}

\begin{theorem}[Any $\boldsymbol{\mathbf{\alpha}}$]
\label{thm:main-halfPerfect-informal} \label{thm:anyAlpha}
Suppose the fraction of the perfect labelers  is $\alpha$ and let $\delta' = c \alpha \delta$ for small enough constant $c>0$. Algorithm~\ref{alg:anyAlpha} uses oracle $\O_\F$, runs in time $\poly(d, \frac 1\alpha,  \frac 1\epsilon, \ln(\frac 1\delta))$, uses a training set of size $O(\frac 1\alpha m_{\epsilon,  \delta'})$ size and with probability $1-\delta$ returns $f\in \F$ with $\err_D(f) \leq \epsilon$ using $O(\frac 1 {\alpha})$ golden queries, load of $\frac 1\alpha $ per labeler, and a total number of queries
\[  O\left( 
 \frac 1\alpha m_{\epsilon, \delta'} + \frac{1}{\alpha\epsilon} \log(\frac{1}{\delta'}) \log(\frac{1}{\epsilon \delta'})
 + \frac{1}{\alpha^3} m_{\sqrt{\epsilon}, \delta'} \log(\frac{m_{\sqrt \epsilon, \delta'}}{\delta'})
  \right).
\]
Note that when  $\frac{1}{\alpha^2\sqrt \epsilon} \geq  \log\left(\frac{m_{\sqrt{\epsilon}, \delta}}{\alpha \delta}  \right)$ and $\log(\frac{1}{\alpha \delta}) < d$, the cost per labeled query is $O(\frac 1\alpha)$.
\end{theorem}

\begin{proof}[{\bf Proof Sketch}]
Let $B = \{ x \mid \smaj_P(x)\leq 1- \alpha/2\}$ be the set of good test cases and  let $\beta = D[B]$ be the total density on such points.
Note that if $\beta > \frac \epsilon 4$, with high probability $S_0$ includes one such point, in which case $\textsc{Prune-and-Label}$ identifies it and prunes the set of labelers. Therefore, we can assume that $\beta \leq \frac \epsilon 4$.

By Lemma~\ref{lem:test-and-label}, it is easy to see that Phase 1 and Phase 3 of Algorithm~\ref{alg:anyAlpha} succeed in producing $h_1$ and $h_3$ such that $err_D(h_1) \leq  \frac 12 \sqrt{\epsilon}$ and $err_{D_3}(h_3) \leq \frac 12 \sqrt{\epsilon}$. 
It remains to show that Phase 2 of Algorithm~\ref{alg:anyAlpha} also produces $h_2$ such that $err_{D_2}(h_2) \leq \frac 12  \sqrt{\epsilon}$.

Consider the filtering step of Phase 2. First note that 
for any  $x\notin B$, the per-point guarantees of \textsc{Filter} expressed in Lemma~\ref{lem:filter} still hold.
Let $\D'$ be the distribution that has equal probability on the distributions induced by $\bar{W_I}$ and $\bar{W_C}$, and is used for simulating $\D_2$.
Similarly as in Lemma~\ref{lem:h_2-err} one can show that for any $x\not\in B$, $d'(x) = \Theta(d_2(x))$.
Since $\D[B]\leq \frac \epsilon 4$, we have that $\D_2[B]\leq \frac 14 \sqrt{\epsilon} $. 
Therefore, $\D'$ and $\D_2$ satisfy the conditions of the robust super-sampling lemma (Lemma~\ref{lem:super-sample-robust}) where the fraction of bad points is at most $\frac {\sqrt{\epsilon}} 4$. Hence, we can argue that $err_{\D_2}(h_2) \leq \frac {\sqrt{\epsilon}} {2}$.

The remainder of the proof follows by using the boosting technique of \cite{schapire1990strength} described in Theorem~\ref{thm:schapire}.

\end{proof}

%% file: community.tex
\section{No Perfect Labelers} \label{sec:detect}

In this section, we consider a scenario where our pool of labelers does not include any perfect labelers.
Unfortunately, learning  $f^*$ in this setting reduces to the notoriously difficult agnostic learning problem.
A related task is to find a set of the labelers which are \emph{good} but not perfect.
In this section, we show how to identify the set of all good labelers, when at least the majority of the labelers are good.

We consider a setting where the fraction of the perfect labelers, $\alpha$,  is arbitrarily small or $0$.
We further assume that at least half of the labelers are good, while others have considerably worst performance.
More formally, we are given a set of labelers $\ell_1, \dots, \ell_n$ and a distribution $D$ with an unknown target classifier $f^*\in \F$.
We assume that more than half of these labelers are ``good'', that is they  have error of $\leq \epsilon$ on distribution $D$. On the other hand, the remaining labelers, which we call ``bad'',  have error rates $\geq 4\epsilon$ on distribution $D$.
We are interested in identifying all of the good labelers with high probability by querying the labelers on an unlabeled  sample set drawn from $\DX$.

This model presents an interesting community structure: Two good labelers agree on at least $1-2\epsilon$ fraction of the data, while a bad and a good labeler agree on at most  $1-3\epsilon$ of the data. Note that the rate of agreement between two bad labelers  can be arbitrary.
This is due to the fact that there can be multiple bad labelers with the same classification function, in which case they completely agree with each other, or two bad labelers who disagree on the classification of every instance.
This structure serves as the basis of  Algorithm~\ref{alg:good-detection} and its analysis.
Here we provide an overview of its working and analysis.

\begin{algorithm}[h]
\SetAlgoNoLine
\KwIn{Given $n$ labelers, parameters $\epsilon$ and $\delta$}
Let $G= ([n], \emptyset)$ be a graph on $n$ vertices with no edges.\\
Take set $Q$ of $16\ln(2)n$ random pairs of nodes from $G$.\\
\nl  \label{item:random-test} \For{$(i,j) \in Q$}{ 
   \lIf{$\textsc{disagree}(i,j) <  2.5 \epsilon$}{add edge $(i,j)$ to $G$}
}   
\nl Let $\mathcal{C}$ be the set of connected components of $G$ each with $\geq n/4$ nodes.\label{item:connect}\\
\nl \label{item:connect-small}\For{$i \in [n] \setminus \left( \bigcup_{C\in  \mathcal{C}} C \right)$ and $C\in \mathcal{C}$}{Take one node $j\in C$, if $\textsc{disagree}(i,j) < 2.5 \epsilon$ add edge $(i,j)$ to $G$.}
\Return{The largest connected component of $G$}\\
\hrule\vspace*{2pt}
\textbf{$\textsc{disagree}(i,j)$:}
\vspace*{2pt}\hrule\vspace*{2pt}
Take set $S$ of $\Theta(\frac 1\epsilon \ln(\frac n \delta))$ samples from $D$. \\
\Return {$\frac{1}{|S|}\sum_{x\in S} \mathds{1}_{ ( \ell_i(x) \neq \ell_j(x) )}$.}
\caption{\textsc{Good Labeler Detection}}
\label{alg:good-detection}
\end{algorithm}

\begin{theorem}[Informal]
\label{thm:main-noPerfect-informal}

Suppose that any \emph{good} labeler $i$ is such that $\err_D(g_i) \leq \epsilon$. Furthermore, assume that  $\err_D(g_j) \not\in (\epsilon, 4\epsilon)$ for any $j\in[n]$. And let the number of good labelers be at least $\lfloor \frac n2 \rfloor +1$. Then, Algorithm~\ref{alg:good-detection}, returns the set of all good labeler with probability $1-\delta$, using an expected load of
$\load = O\left(\frac 1\epsilon \ln \left( \frac n \delta \right) \right)$ per labeler.

\end{theorem}

We view the labelers as nodes in a graph that has no edges at the start of the algorithm.
In step~\ref{item:random-test}, the algorithm takes $O(n)$ random pairs of labelers and estimates their level of disagreement by querying them on an unlabeled sample set of size $O\left( \frac 1 \epsilon \ln\left(\frac n \delta \right) \right)$ and measuring their empirical disagreement.
By an application of Chernoff bound, we know that with probability $1-\delta$, for any $i,j\in[n]$,
\[ \left|  \textsc{Disagree}(i,j) - \Pr_{x\sim \DX} [\ell_i(x) \neq \ell_j(x)  ]\right| < \frac \epsilon 2.
\]
Therefore, for any pair of good labelers $i$ and $j$ tested by the algorithm, $\textsc{Disagree}(i,j)< 2.5 \epsilon$, and for any pair of labelers $i$ and $j$ that one is good and the other is bad, $\textsc{Disagree}(i,j)\geq 2.5 \epsilon$.
Therefore, the connected components of such a graph only include labelers from a single community.

Next, we show that at step~\ref{item:connect} of Algorithm~\ref{alg:good-detection} with probability $1-\delta$ there exists at least one connected component of size $n/4$ of good labelers.

To see this we first prove that for any two good labelers $i$ and $j$, the probability of $(i,j)$ existing is at least $ \Theta(1/n)$. 
Let $V_g$ be the set of nodes corresponding to good labelers. For $i, j\in V_g$, we have
\[
\Pr[(i,j) \in G] = 1 - \left( 1- \frac{1}{n^2} \right)^{4\ln(2) n} \approx \frac{4 \ln(2)}{n} \geq \frac{2 \ln(2)}{|V_g|}.
\]
By the properties of random graphs, with very high probability there is a component of size $\beta |V_g|$ in a random graph whose edges exists with probability $c/|V_g|$, for $\beta + e^{-\beta c} =1$~\citep{janson2011random}. Therefore, with probability $1-\delta$, there is a component of size $|V_g|/2 > n/4$ over the vertices in $V_g$.

Finally, at step~\ref{item:connect-small} the algorithm considers smaller connected components and tests whether they join any of the bigger components, by measuring the disagreement of two arbitrary labelers from these components.,At this point, all good labelers form one single connected component of size $> \frac n2$. So, the algorithm succeeds in identifying all good labelers. 

Next, we briefly discuss the expected  load per labeler in Algorithm~\ref{alg:good-detection}. 
Each labeler participates in $O(1)$ pairs of disagreement tests in expectation, each requiring $O(\frac 1\epsilon \ln(n/\delta))$ queries. 
So, in expectation each labeler labels  $O(\frac 1\epsilon \ln(n/\delta))$ instances.

%% file: app_related.tex
\section{Additional Related Work}  \label{app:related}

More generally, interactive models of learning have been studied in the machine learning community. The most popular among them is the area of {active learning}~\citep{Cohn94, Dasgupta05, BBL06, Kol10, hanneke:11}. In this model, the learning algorithm can adaptively query for the labels of a few examples in the training set and use them to produce an accurate hypothesis. The goal is to use as few label queries as possible. The number of labeled queries used is called the {\em label complexity} of the algorithm. It is known that certain hypothesis classes can be learned in this model using much fewer labeled queries than predicted by the VC theory. In particular, in many instances the label complexity scales only logarithmically in $\frac 1 {\epsilon}$ as opposed to linearly in $\frac 1 {\epsilon}$. However, to achieve computational efficiency, the algorithms in this model rely on the fact that one can get perfect labels for every example queried. This would be hard to achieve in our model since in the worst case it would lead to each labeler answering $\log (\frac d {\epsilon})$ many queries. In contrast, we want to keep the query load of a labeler to a constant and hence the techniques developed for active learning are insufficient for our purposes. Furthermore, in noisy settings most work on efficient active learning algorithms assumes the existence of an {\em empirical risk minimizer}~(ERM) oracle that can minimize training error even when the instances aren't labeled according to the target classifier. However, in most cases such an ERM oracle is hard to implement and  the improvements obtained in the label complexity are less drastic in such noisy scenarios.

Another line of work initiated by~\cite{zhang2015active} models related notions of weak and strong labelers in the context of active learning. The authors study scenarios where the label queries to the strong labeler can be reduced by querying the weak and potentially noisy labelers more often. However, as discussed above, the model does not yield relevant  algorithms for our setting as in the worst case one might end up querying for $\frac d {\epsilon}$ high quality labels leading to a prohibitively large load per labeler in our setting. The work of~\cite{yan2016active} studies a model of active learning where the labeler abstains from providing a label prediction more often on instances that are closer to the decision boundary.
The authors then show how to use the abstentions in order to approximate the decision boundary. Our setting is inherently different, since we make no assumptions on the bad labelers.

%% file: app_super-sample.tex
\section{Proof of Lemma~\ref{lem:super-sample}}\label{app:super-sample}

First, notice that because $D$ and $D'$ are both labeled according to $f^*\in \F$, for any $f\in \F$ we have,
\[ \err_{D'}(f) = \sum_{x} d'(x) \mathds{1}_{f(x) \neq f^*(x)} \geq  \sum_{x} c\cdot d(x) \mathds{1}_{f(x) \neq f^*(x)}  = c \cdot \err_D(f).
\]
Therefore, if $\err_{D'}(f) \leq c \epsilon$, then $\err_D(f) \leq \epsilon$. 
Let $m' = m_{c \epsilon, \delta}$, we have
\begin{align*}
\delta &> \Pr_{S'\sim D'^{m'}} [\exists f\in \F, \text{s.t. }\err_{S'}(f) = 0 \wedge \err_{D'}(f)\geq c \epsilon] \\
&  \geq \Pr_{S'\sim D'^{m'}} [\exists f\in \F, \text{s.t. }\err_{S'}(f) = 0 \wedge \err_{D}(f)\geq \epsilon].
\end{align*}
The claim follows by the fact that $m_{c \epsilon, \delta}=O\left( \frac 1c m_{\epsilon, \delta} \right)$.

%% file: app_size.tex
\section{Proof of Lemma~\ref{lem:size}} \label{app:size}

Let us first consider the expected size of sets $S_I$,  $\bar{W}_I$, and $\bar{W}_C$. Using Lemma~\ref{lem:filter}, we have
\[  O(m_{\sqrt \epsilon, \delta}) \geq \frac 12 \sqrt\epsilon |S_2| + \sqrt{\epsilon} |S_2|   \geq   \E[|S_I|] \geq \frac 12 \left( \frac 12 \sqrt\epsilon\right) |S_2| \geq \Omega(m_{\sqrt \epsilon, \delta}).
\]
Similarly, 
\[ O(m_{\sqrt \epsilon, \delta}) \geq \E[S_I] + |S_C| \geq \E[\bar{W}_I] \geq  \frac 12 \left( \frac 12 \sqrt\epsilon\right) |S_2| \geq \Omega(m_{\sqrt \epsilon, \delta}).
\]
Similarly,
\[  O(m_{\sqrt \epsilon, \delta}) \geq  \E[S_I]  +  |S_C|  \geq \E[\bar{W}_C] \geq \left(1- \frac 12 \sqrt \epsilon \right) |S_C| \geq \Omega(m_{\sqrt \epsilon, \delta}).
\]

The claim follows by the Chernoff bound.

%% file: app_bernstein.tex
\section{Remainder of the Proof of Lemma~\ref{lem:h_2-label}} \label{app:h_2-label_bernstein}

We prove that  the Bernstein inequality holds for the total number of queries  $y_1, y_2, \dots,$ made before their majority agrees  with $f^*(x)$. 
Let $L_x$ be the random variable denoting the number of queries the algorithm makes on instance $x$ for which $h(x) = f^*(x)$. 
Consider the probability that $L_x = 2k+1$ for some $k$. That is,  $\maj(y_{1:t}) = f^*(x)$ for the first time when $t = 2k+1$. 
This is at most the probability that $\maj(y_{1:2k-1}) \neq f^*(x)$.
By Chernoff bound, we have that
\begin{align*}
\Pr[L_x = 2k+1] &\leq \Pr[ \maj(y_{1:2k-1}) \neq f^*(x) ] \leq \exp\left(- 0.7(2k-1) (\frac 27)^2 / 2 \right) \\
& \leq \exp\left( -0.02 (2k-1) \right). 
\end{align*}
For each $i>1$, we have 
\begin{align*}
\E[(L_x - \E[L_x])^i] &\leq  \sum_{k=0}^\infty \Pr[L_x = 2k+1] (2k+1  - \E[L_x])^i  \\
&\leq \sum_{k=0}^\infty  e^{-0.02 (2k-1)} (2k+1)^i \\
&\leq  e^{0.04}\sum_{k=0}^\infty  e^{-0.02 (2k+1)} (2k+1)^i \\
&\leq  e^{0.04}\sum_{k=0}^\infty  e^{-0.02 k} k^i \\
&\leq 50 (i+1)! \, e^{4i + 0.04},
\end{align*}
where the last inequality is done by integration. This satisfies the Bernstein condition stated in Theorem~\ref{thm:bernstein}. Therefore, 
\begin{align*}
\Pr\left[ \sum_{x\in S} L_x - |S|\E[L_x] \geq  O(|S|)] \right] \leq \exp \left(- |S| \right).
\end{align*}
Therefore, the total number of queries over all points in $x\in S$ where $h(x) = f^*(x)$ is at most $O(|S|)$ with very high probability.

%% file: app_probability.tex
\section{Probability Lemmas} \label{app:probability}

\begin{theorem}[Probability of Ruin~\citep{feller2008introduction}] \label{thm:ruin}
Consider a player who starts with $i$ dollars  against an adversary that has $N$ dollars. The player bets one dollar in each gamble, which he wins with probability $p$. The probability that the player ends up with no money at any point in the game is 
\[\frac{1 - \left( \frac{p}{1-p} \right)^{N} } {1 - \left( \frac{p}{1-p} \right)^{N+i} }.
\]
\end{theorem}

\begin{theorem}[Bernstein Inequality] \label{thm:bernstein}
Let $X_1, \dots, X_n$ be independent random variables with expectation $\mu$. Supposed that for some positive real number $L$ and every $k >1$,
\[ \E[ (X_i - \mu)^k]  \leq \frac 12 \E[(X_i- \mu)^2] L^{k-2} k!.\]
Then, 
\[  \Pr\left[ \sum_{i=1}^n X_i - n \mu \geq 2t \sqrt{\sum_{i=1}^n \E[(X_i- \mu)^2]}  \right] < \exp(-t^2), \quad \text{ for } 0 < t \leq \frac{1}{2L} \sqrt{\E[(X_i- \mu)^2]}.
\]
\end{theorem}

%% file: anyAlphaApp.tex
\section{Omitted Proofs from Section~\ref{sec:boost-alpha}}
  \label{sec:boost-alpha-app}

In this section, we prove Theorem~\ref{thm:main-halfPerfect-informal} and present the proofs that were omitted from Section~\ref{sec:boost-alpha}.

\medskip
\noindent \textbf{Theorem~\ref{thm:main-halfPerfect-informal}~(restated)}~\emph{
Suppose the fraction of the perfect labelers  is $\alpha$ and let $\delta' = \Theta(\alpha \delta)$. Algorithm~\ref{alg:anyAlpha} uses oracle $\O_\F$, runs in time $\poly(d, \frac 1\alpha,  \frac 1\epsilon, \ln(\frac 1\delta))$, uses a training set of size $O(\frac 1\alpha m_{\epsilon,  \delta'})$ size and with probability $1-\delta$ returns $f\in \F$ with $\err_D(f) \leq \epsilon$ using $O(\frac 1 {\alpha})$ golden queries, load of $\frac 1\alpha $ per labeler, and a total number of queries
\[  O\left( 
 \frac 1\alpha m_{\epsilon, \delta'} + \frac{1}{\alpha\epsilon} \log(\frac{1}{\delta'}) \log(\frac{1}{\epsilon \delta'})
 + \frac{1}{\alpha^3} m_{\sqrt{\epsilon}, \delta'} \log(\frac{m_{\sqrt \epsilon, \delta'}}{\delta'})
  \right).
\]
Note that when  $\frac{1}{\alpha^2\sqrt \epsilon} \geq  \log\left(\frac{m_{\sqrt{\epsilon}, \delta}}{\alpha \delta}  \right)$ and $\log(\frac{1}{\alpha \delta}) < d$, the cost per labeled query is $O(\frac 1\alpha)$.
}
\medskip

\subsection{Proof of Lemma~\ref{lem:test-and-label}} \label{app:test-and-label-proof}

By Chernoff bound, with probability $\geq 1-\delta$, for every $x \in S$ we have that 
\[
\left| \smaj_{P}(x) - \smaj_{L}(x)   \right| \leq \frac \alpha 8,
\]
where $L$ is the set of labelers $\textsc{Prune-and-Label}(S, \delta)$ queries on $x$.
Hence, if $x$ is such that  $\smaj_P(x) \leq 1 - \frac \alpha 2$, then it will be identified and the set of labelers is pruned. Otherwise, $\maj_L(x)$ agrees with the good labelers and $x$ gets labeled correctly according to the target function.  
\subsection{Proof of Lemma~\ref{lem:depth-golden}} \label{app:depth-golden-proof}
Recall that $\delta' = c\cdot \alpha \delta$ for some  small  enough constant $c>0$.
Each time $\textsc{Prune-and-Label}(S, \delta')$ is called, by Hoeffding bound, it is guaranteed that with probability $\geq 1-\delta'$, for each $x \in S$,
\[
\left| \smaj_{P}(x) - \smaj_{L}(x)   \right| \leq \frac \alpha 8,
\]
where $L$ is the set of labelers $\textsc{Prune-and-Label}(S, \delta')$ queries on $x$.
Hence, when we issue a golden query for $x$ such that $\smaj_{L}(x) \leq 1 - \frac \alpha 4$ and prune away bad labelers, we are guaranteed to remove at least an $\frac \alpha 8$ fraction of the labelers. 
 Furthermore, no good labeler is ever removed. Hence, the fraction of good labelers increases from $\alpha$ to $\alpha / (1- \frac \alpha 8)$. So, in $O(\frac 1 \alpha)$ calls, the fraction of the good labelers surpasses $\frac 3 4$ and we switch to using Algorithm~\ref{alg:boost}. Therefore, with probability $1-\delta$ overall, the total number of golden queries is $O(1/\alpha)$.

\subsection{Proof of Lemma~\ref{lem:super-sample-robust}}  \label{app:uper-sample-robust-proof}

Let $B$ be the set of points that do not satisfy the condition that $d'(x) \geq c\cdot d(x)$. Notice that because $D$ and $D'$ are both labeled according to $f^*\in \F$, for any $f\in \F$ we have,
\[ \err_{D'}(f) = \sum_{x \in B} d'(x) \mathds{1}_{f(x) \neq f^*(x)} + \sum_{x \notin B} d'(x) \mathds{1}_{f(x) \neq f^*(x)} \geq  \sum_{x \notin B} c\cdot d(x) \mathds{1}_{f(x) \neq f^*(x)}  \geq c \cdot (\err_D(f) - \epsilon).
\]
Therefore, if $\err_{D'}(f) \leq c \epsilon$, then $\err_D(f) \leq 2\epsilon$. 
Let $m' = m_{c \epsilon, \delta}$, we have
\begin{align*}
\delta &> \Pr_{S'\sim D'^{m'}} [\exists f\in \F, \text{s.t. }\err_{S'}(f) = 0 \wedge \err_{D'}(f)\geq c \epsilon] \\
&  \geq \Pr_{S'\sim D'^{m'}} [\exists f\in \F, \text{s.t. }\err_{S'}(f) = 0 \wedge \err_{D}(f)\geq 2\epsilon].
\end{align*}
The claim follows by the fact that $m_{c \epsilon, \delta}=O\left( \frac 1c m_{\epsilon, \delta} \right)$.

\subsection{Proof of Theorem~\ref{thm:anyAlpha}}
Recall that $\delta' = c \cdot \alpha \delta$ for a small enough constant $c>0$.  
Let $B = \{ x \mid \smaj_P(x)\leq 1- \alpha/2\}$ be the set of good test cases and and let $\beta = D[B]$ be the total density on such points.
Note that if $\beta > \frac \epsilon 4$, with high probability $S_0$ includes one such point, in which case $\textsc{Prune-and-Label}$ identifies it and prunes the set of labelers. Therefore, we can assume that $\beta \leq \frac \epsilon 4$.
By Lemma~\ref{lem:test-and-label}, it is easy to see that  $err_D(h_1) \leq\frac 12 \sqrt{\epsilon}$.

We now analyze the filtering step of Phase 2. As in Section~\ref{sec:boost}, our goal is to argue that $err_{D_2}(h_2) \leq \frac 12 \sqrt{\epsilon}$. 
Consider distribution $\D'$ that has equal probability on the distributions induced by $\bar{W_I}$ and $\bar{W_C}$ and let $d'(x)$ denote the density of point $x$ in this distribution. We will show that for any $x\notin B$ we have that $d'(x) = \Theta(d_2(x))$. 
Since $\D[B]\leq \frac \epsilon 4$, we have that $\D_2[B]\leq \frac 14 \sqrt{\epsilon}$. 
Therefore, $\D'$ and $\D_2$ satisfy the conditions of the robust super-sampling lemma (Lemma~\ref{lem:super-sample-robust}) where the fraction of bad points is at most $\frac {\sqrt{\epsilon}} 4$. Hence, $err_{\D_2}(h_2) \leq \frac 12 \sqrt{\epsilon}$.

We now show that for any $x\in B$,  $d'(x) = \Theta(d_2(x))$. The proof is identical to the one in Lemma~\ref{lem:h_2-err}. For ease of representation, we assume that $\err_{\D}(h_1)$ is exactly $\frac 12 \sqrt{\epsilon}$.
Let $d(x)$, $d_2(x)$, $d_C(x)$, and $d_I(x)$ be the density of instance $x$ in distributions $\D$, $\D_2$, $\D_C$, and $\D_I$, respectively. Note that, for any $x$ such that $h_1(x) = f^*(x)$, we have $d(x) = d_C(x) (1-  \frac 12\sqrt{\epsilon})$. Similarly, for any $x$ such that $h_1(x) \neq f^*(x)$, we have $d(x) = d_I(x)\frac 12 \sqrt{\epsilon}$.
Let $N_C(x)$, $N_I(x)$, $M_C(x)$ and $M_I(x)$  be the number of occurrences of $x$ in the sets $S_C$, $S_I$, $\bar{W_C}$ and $\bar{W_I}$, respectively.  
For any $x$, there are two cases:

\medskip
\noindent{If $h_1(x) = f^*(x)$:} Then, there exist absolute constants $c_1$ and $c_2$ according to Lemma~\ref{lem:size}, such that
\begin{align*}
d'(x)&= \frac 12 \E\left[ \frac{M_C(x)}{|\bar{W_C}|} \right] \geq \frac{\E[M_C(x)]}{ c_1 \cdot  m_{\sqrt \epsilon, \delta} } \geq \frac{\E[N_C(x)]}{c_1 \cdot  m_{\sqrt \epsilon, \delta} } = \frac{| S_C | \cdot d(x) }{c_1 \cdot  m_{\sqrt \epsilon, \delta} } \\
& = \frac{| S_C | \cdot d_C(x) \cdot (1- \frac 12 \sqrt{\epsilon}) }{c_1 \cdot  m_{\sqrt \epsilon, \delta} } \geq 
c_2 d_C(x) = \frac{c_2 d_2(x)}{2},
\end{align*}
where the second and sixth transitions are by the sizes of $\bar{W_C}$ and $|S_C|$ and the third transition is by the fact that if $h(x) = f^*(x)$, $M_C(x) > N_C(x)$.\\

\medskip
\noindent{If $h_1(x) \neq f^*(x)$:} Then, there exist absolute constants $c'_1$ and $c'_2$ according to Lemma~\ref{lem:size}, such that
\begin{align*}
d'(x) &= \frac 12 \E\left[ \frac{M_I(x)}{|\bar{W_I}|} \right] \geq \frac{\E[M_I(x)]}{c'_1 \cdot  m_{\sqrt \epsilon, \delta} }
 \geq \frac{\E[N_I(x)]}{c'_1 \cdot  m_{\sqrt \epsilon, \delta} }\geq \frac{ 0.5~  d(x) | S_2|}{c'_1 \cdot  m_{\sqrt \epsilon, \delta} }\\
& = \frac{0.5~  d_I(x) \frac 12 \sqrt{\epsilon} \cdot | S_2| }{c'_1 \cdot  m_{\sqrt \epsilon, \delta} } = c'_2 d_I(x) = \frac{c'_2 d_2(x)}{2},
\end{align*}
where the second and sixth transitions are by the sizes of $\bar{W_I}$ and $|S_2|$,
the third transition is by the fact that if $h(x) \neq f^*(x)$, $M_I(x) > N_I(x)$, and the fourth transition holds by part 2 of Lemma~\ref{lem:filter}.

Finally, we have that $\err_{\D_3}(h_3) \leq \frac 12 \sqrt{\epsilon}$, where $\D_3$ is distribution $\D$ conditioned on $\{x\mid h_1(x) \neq h_2(x)\}$. Using the boosting technique of \cite{schapire1990strength} describe in Theorem~\ref{thm:schapire}, we  conclude that 
$\Maj(h_1, h_2, h_3)$ has error $\leq \epsilon$ on $\D$.

The label complexity claim follows by the fact that we restart Algorithm~\ref{alg:anyAlpha} at most $O(1/\alpha)$ times, take an additional $O(\frac 1\epsilon \log(\frac{1}{\delta'}))$ high quality labeled set, and each run of Algorithm~\ref{alg:anyAlpha} uses the same label complexity as in Theorem~\ref{thm:main-manyPerfect-informal} before getting restarted.

%% file: crowd.bbl
\begin{thebibliography}{31}
\providecommand{\natexlab}[1]{#1}
\providecommand{\url}[1]{\texttt{#1}}
\expandafter\ifx\csname urlstyle\endcsname\relax
  \providecommand{\doi}[1]{doi: #1}\else
  \providecommand{\doi}{doi: \begingroup \urlstyle{rm}\Url}\fi

\bibitem[Anthony and Bartlett(1999)]{AB99}
M.~Anthony and P.~L. Bartlett.
\newblock \emph{Neural Network Learning: Theoretical Foundations}.
\newblock Cambridge University Press, 1999.

\bibitem[Awasthi et~al.(2015)Awasthi, Balcan, Haghtalab, and Urner]{ABHU15}
Pranjal Awasthi, Maria~Florina Balcan, Nika Haghtalab, and Ruth Urner.
\newblock Efficient learning of linear separators under bounded noise.
\newblock In \emph{Proceedings of the 28th Conference on Computational Learning
  Theory (COLT)}, pages 167--190, 2015.

\bibitem[Awasthi et~al.(2016)Awasthi, Balcan, Haghtalab, and
  Zhang]{awasthi2016learning}
Pranjal Awasthi, Maria-Florina Balcan, Nika Haghtalab, and Hongyang Zhang.
\newblock Learning and 1-bit compressed sensing under asymmetric noise.
\newblock In \emph{Proceedings of the 29th Conference on Computational Learning
  Theory (COLT)}, pages 152--192, 2016.

\bibitem[Badanidiyuru et~al.(2012)Badanidiyuru, Kleinberg, and Singer]{BKS12}
Ashwinkumar Badanidiyuru, Robert Kleinberg, and Yaron Singer.
\newblock Learning on a budget: posted price mechanisms for online procurement.
\newblock In \emph{Proceedings of the13thACM Conference on Economics and
  Computation (EC)}, pages 128--145. ACM, 2012.

\bibitem[Badanidiyuru et~al.(2013)Badanidiyuru, Kleinberg, and Slivkins]{BKS13}
Ashwinkumar Badanidiyuru, Robert Kleinberg, and Aleksandrs Slivkins.
\newblock Bandits with knapsacks: Dynamic procurement for crowdsourcing.
\newblock In \emph{The 3rd Workshop on Social Computing and User Generated
  Content, co-located with ACM EC}, 2013.

\bibitem[Balcan et~al.(2006)Balcan, Beygelzimer, and Langford]{BBL06}
Maria-Florina Balcan, A.~Beygelzimer, and J.~Langford.
\newblock Agnostic active learning.
\newblock In \emph{Proceedings of the23rdInternational Conference on Machine
  Learning (ICML)}, pages 65--72. ACM, 2006.

\bibitem[Bousquet et~al.(2005)Bousquet, Boucheron, and Lugosi]{bbl05}
O.~Bousquet, S.~Boucheron, and G.~Lugosi.
\newblock Theory of classification: a survey of recent advances.
\newblock \emph{ESAIM: Probability and Statistics}, 9:\penalty0 323--375, 2005.

\bibitem[Cohn et~al.(1994)Cohn, Atlas, and Ladner]{Cohn94}
D.~Cohn, L.~Atlas, and R.~Ladner.
\newblock Improving generalization with active learning.
\newblock \emph{Machine Learning}, 15\penalty0 (2), 1994.

\bibitem[Dasgupta(2005)]{Dasgupta05}
Sanjoy Dasgupta.
\newblock Coarse sample complexity bounds for active learning.
\newblock In \emph{Proceedings of the 19th Annual Conference on Neural
  Information Processing Systems (NIPS)}, 2005.

\bibitem[Dekel and Shamir(2009)]{dekel2009vox}
Ofer Dekel and Ohad Shamir.
\newblock Vox populi: Collecting high-quality labels from a crowd.
\newblock In \emph{Proceedings of the 22nd Conference on Computational Learning
  Theory (COLT)}, 2009.

\bibitem[Feller(2008)]{feller2008introduction}
Willliam Feller.
\newblock \emph{An introduction to probability theory and its applications},
  volume~2.
\newblock John Wiley \& Sons, 2008.

\bibitem[Freund(1990)]{freund1990boosting}
Yoav Freund.
\newblock Boosting a weak learning algorithm by majority.
\newblock In \emph{Proceedings of the 22nd Conference on Computational Learning
  Theory (COLT)}, volume~90, pages 202--216, 1990.

\bibitem[Freund and Schapire(1995)]{freund1995desicion}
Yoav Freund and Robert~E Schapire.
\newblock A desicion-theoretic generalization of on-line learning and an
  application to boosting.
\newblock In \emph{European conference on computational learning theory}, pages
  23--37. Springer, 1995.

\bibitem[Hanneke(2011)]{hanneke:11}
S.~Hanneke.
\newblock Rates of convergence in active learning.
\newblock \emph{The Annals of Statistics}, 39\penalty0 (1):\penalty0 333--361,
  2011.

\bibitem[Ho et~al.(2013)Ho, Jabbari, and Vaughan]{vaughan2013adaptive}
Chien-Ju Ho, Shahin Jabbari, and Jennifer~Wortman Vaughan.
\newblock Adaptive task assignment for crowdsourced classification.
\newblock \emph{Proceedings of the 30th International Conference on Machine
  Learning (ICML)}, 2013.

\bibitem[Ipeirotis et~al.(2010)Ipeirotis, Provost, and
  Wang]{ipeirotis2010quality}
Panagiotis~G Ipeirotis, Foster Provost, and Jing Wang.
\newblock Quality management on amazon mechanical turk.
\newblock In \emph{Proceedings of the International Conference on Knowledge
  Discovery and Data Mining (KDD)}, pages 64--67. ACM, 2010.

\bibitem[Janson et~al.(2011)Janson, Luczak, and Rucinski]{janson2011random}
Svante Janson, Tomasz Luczak, and Andrzej Rucinski.
\newblock \emph{Random graphs}, volume~45.
\newblock John Wiley \& Sons, 2011.

\bibitem[Karger et~al.(2011)Karger, Oh, and Shah]{karger2011iterative}
David~R Karger, Sewoong Oh, and Devavrat Shah.
\newblock Iterative learning for reliable crowdsourcing systems.
\newblock In \emph{Proceedings of the 25th Annual Conference on Neural
  Information Processing Systems (NIPS)}, pages 1953--1961, 2011.

\bibitem[Karger et~al.(2014)Karger, Oh, and Shah]{karger2014budget}
David~R Karger, Sewoong Oh, and Devavrat Shah.
\newblock Budget-optimal task allocation for reliable crowdsourcing systems.
\newblock \emph{Operations Research}, 62\penalty0 (1):\penalty0 1--24, 2014.

\bibitem[Kittur et~al.(2008)Kittur, Chi, and Suh]{kittur2008crowdsourcing}
Aniket Kittur, Ed~H Chi, and Bongwon Suh.
\newblock Crowdsourcing user studies with mechanical turk.
\newblock In \emph{Proceedings of the SIGCHI conference on human factors in
  computing systems}, pages 453--456. ACM, 2008.

\bibitem[Koltchinskii(2010)]{Kol10}
V.~Koltchinskii.
\newblock Rademacher complexities and bounding the excess risk in active
  learning.
\newblock \emph{Journal of Machine Learning Research}, 11:\penalty0 2457--2485,
  2010.

\bibitem[Rivest and Sloan(1994)]{rivest1994formal}
Ronald~L Rivest and Robert Sloan.
\newblock A formal model of hierarchical concept-learning.
\newblock \emph{Information and Computation}, 114\penalty0 (1):\penalty0
  88--114, 1994.

\bibitem[Schapire(1990)]{schapire1990strength}
Robert~E Schapire.
\newblock The strength of weak learnability.
\newblock \emph{Machine learning}, 5\penalty0 (2):\penalty0 197--227, 1990.

\bibitem[Singla and Krause(2013)]{singla2013truthful}
Adish Singla and Andreas Krause.
\newblock Truthful incentives in crowdsourcing tasks using regret minimization
  mechanisms.
\newblock In \emph{Proceedings of the 22nd international conference on World
  Wide Web}, pages 1167--1178. ACM, 2013.

\bibitem[Slivkins and Vaughan(2014)]{slivkins2014online}
Aleksandrs Slivkins and Jennifer~Wortman Vaughan.
\newblock Online decision making in crowdsourcing markets: Theoretical
  challenges.
\newblock \emph{ACM SIGecom Exchanges}, 12\penalty0 (2):\penalty0 4--23, 2014.

\bibitem[Steinhardt et~al.(2016)Steinhardt, Valiant, and
  Charikar]{steinhardt2016avoiding}
Jacob Steinhardt, Gregory Valiant, and Moses Charikar.
\newblock Avoiding imposters and delinquents: Adversarial crowdsourcing and
  peer prediction.
\newblock In \emph{Proceedings of the 30th Annual Conference on Neural
  Information Processing Systems (NIPS)}, pages 4439--4447, 2016.

\bibitem[Tran-Thanh et~al.(2014)Tran-Thanh, Stein, Rogers, and
  Jennings]{tran2014efficient}
Long Tran-Thanh, Sebastian Stein, Alex Rogers, and Nicholas~R Jennings.
\newblock Efficient crowdsourcing of unknown experts using bounded multi-armed
  bandits.
\newblock \emph{Artificial Intelligence}, 214:\penalty0 89--111, 2014.

\bibitem[Valiant(1984)]{Valiant:84}
L.~G. Valiant.
\newblock A theory of the learnable.
\newblock \emph{Communications of the ACM}, 27\penalty0 (11):\penalty0
  1134--1142, 1984.

\bibitem[Wais et~al.(2010)Wais, Lingamneni, Cook, Fennell, Goldenberg, Lubarov,
  Marin, and Simons]{wais2010towards}
Paul Wais, Shivaram Lingamneni, Duncan Cook, Jason Fennell, Benjamin
  Goldenberg, Daniel Lubarov, David Marin, and Hari Simons.
\newblock Towards building a high-quality workforce with mechanical turk.
\newblock \emph{Presented at the NIPS Workshop on Computational Social Science
  and the Wisdom of Crowds}, pages 1--5, 2010.

\bibitem[Yan et~al.(2016)Yan, Chaudhuri, and Javidi]{yan2016active}
Songbai Yan, Kamalika Chaudhuri, and Tara Javidi.
\newblock Active learning from imperfect labelers.
\newblock In \emph{Proceedings of the 30th Annual Conference on Neural
  Information Processing Systems (NIPS)}, pages 2128--2136, 2016.

\bibitem[Zhang and Chaudhuri(2015)]{zhang2015active}
Chicheng Zhang and Kamalika Chaudhuri.
\newblock Active learning from weak and strong labelers.
\newblock In \emph{Proceedings of the 29th Annual Conference on Neural
  Information Processing Systems (NIPS)}, pages 703--711, 2015.

\end{thebibliography}
